\documentclass{article} 
\usepackage{iclr2027_conference,times}

\usepackage{hyperref}
\usepackage{url}

\usepackage{booktabs}       
\usepackage{amsfonts}       
\usepackage{nicefrac}       
\usepackage{microtype}      
\usepackage{xcolor}         

\def\a{{\boldsymbol a}}
\def\B{{\boldsymbol B}}

\def\e{{\boldsymbol e}}

\def\g{{\boldsymbol g}}

\def\h{{\boldsymbol h}}
\def\I{{\boldsymbol I}}

\def\J{{\boldsymbol J}}

\def\p{{\boldsymbol p}}

\def\q{{\boldsymbol q}}

\def\u{{\boldsymbol u}}

\def\v{{\boldsymbol v}}

\def\X{{\boldsymbol X}}
\def\x{{\boldsymbol x}}
\def\Y{{\boldsymbol Y}}
\def\y{{\boldsymbol y}}

\def\z{{\boldsymbol z}}

\def\CM{{\mathcal C}}
\def\DM{{\mathcal D}}
\def\EM{{\mathcal E}}

\def\LM{{\mathcal L}}

\def\SM{{\mathcal S}}

\def\CB{{\mathbb C}}

\def\EB{{\mathbb E}}

\def\PB{{\mathbb P}}

\def\RB{{\mathbb R}}

\usepackage{graphicx}
\usepackage{wrapfig}
\usepackage{amsmath}
\usepackage{amssymb}
\usepackage{algorithm,algorithmicx,algpseudocode}
\usepackage{rotating}
\usepackage{multirow} 
\usepackage{pifont}
\usepackage{dsfont}
\usepackage{colortbl}
\usepackage{tikz}

\newtheorem{theorem}{Theorem}
\newtheorem{proof}{Proof}

\newtheorem{assumption}{Assumption}
\newtheorem{lemma}{Lemma}
\newtheorem{definition}{Definition}
\newtheorem{corollary}{Corollary}
\newtheorem{proposition}{Proposition}

\newcommand\unimodal[1]{\cellcolor{gray!25}#1}

\newcommand{\std}[1]{{\scriptsize$\pm$#1}}
\usepackage{makecell}

\usepackage{titletoc}

\title{Confidence Falls Short: Asymmetric Certainty Gains from Optimization Hinder Multimodal Classification}

\author{\textbf{Longfei Huang}$^{1}$\thanks{Equal contribution.}~~\ \
\textbf{Xiangyu Wu}$^{2*}$\ \
\textbf{Yang Yang}$^{1}$\thanks{Corresponding author.} \\
$^{1}$Nanjing University of Science and Technology\\
$^{2}$Alibaba Group
}

\iclrfinalcopy 
\begin{document}

\maketitle

\begin{abstract}
Multimodal learning (MML) falls into the optimization dilemma due to the modality imbalance phenomenon, leading to suboptimal overall performance in practice. While many attempts primarily focus on balancing the optimization dynamics across modalities to address this issue, we identify a subtle yet critical flaw: optimization yields asymmetric gains in predictive certainty, with the strong modality more confident than the weak one, driving imbalanced modality contributions. In this paper, our analysis reveals that this flaw stems from unimodal characteristics rather than multimodal learning, and this confidence discrepancy can be corrected by positive cross-modal intervention. Based on this insight, we propose multimodal Max Confidence Regularization (MaxCR) to dynamically intervene in modality semantic confidence. Specifically, the semantic confidence of each modality is tracked using a nonlinear sparsity measure. We then design max suppression and max excitation based on this measure to regularize strong and weak modalities, respectively. They penalize and encourage the top-1 confidence, thereby constraining multimodal prediction. To this end, strong and weak modalities are expected to make calibrated confidence, thereby improving the overall performance. Empirical experiments on widely used datasets reveal the superiority of our method through comparison with various state-of-the-art (SOTA) multimodal learning baselines.
\end{abstract}

\section{Introduction}
Multimodal learning \cite{TIL:journals/pami/BaltrusaitisAM19,DOMFN:conf/mm/0074ZGGZ22,MML:conf/iccv/Rakib_2025_ICCV,lee2025generalized,tang2026cccaption,liu2025noisyrollout} aims to effectively integrate heterogeneous information to improve the holistic understanding of tasks, which has received growing attention in various real applications \cite{huang2024refining,xiao2025rebalancing,hao2025uni}. Despite expectations of achieving better performance compared with unimodal approaches, multimodal learning has been surprisingly shown to underperform compared to unimodal ones in certain scenarios \cite{OGM:conf/cvpr/PengWD0H22}. The root of this problem lies in the existence of the modality imbalance phenomenon \cite{OGR-GB:conf/cvpr/WangTF20}, where different modalities converge at different rates.

Fortunately, many studies \cite{OGR-GB:conf/cvpr/WangTF20,OGM:conf/cvpr/PengWD0H22,MLA:conf/cvpr/ZhangYBY24,LFM:conf/nips/0074WJ024,BML:conf/ijcai/ZongDLLZ24,MMPareto:conf/icml/WeiH24,BML:journals/pami/WeiHDW25} have explored this modality imbalance issue from various perspectives. As early pioneering studies, G-Blend \cite{OGR-GB:conf/cvpr/WangTF20} and its successor OGM \cite{OGM:conf/cvpr/PengWD0H22}, AGM \cite{AGM:conf/iccv/LiLHLLZ23}, and PMR \cite{PMR:conf/cvpr/Fan0WW023}, etc., focus on designing customized learning strategies for different modalities to intervene in their optimization processes to achieve rebalancing. Other attempts, including MLA \cite{MLA:conf/cvpr/ZhangYBY24}, DI-MML \cite{DI-MML:conf/mm/FanXWLG24}, and ReconBoost \cite{ReconBoost:conf/icml/CongHua24}, focus on decoupling multimodal learning and bridging the learning gap across modalities by injecting optimization information between them. 

Despite significant progress, existing methods mainly adjust gradients to balance optimization dynamics, while overlooking unimodal intrinsic optimization characteristics. Compared with the strong modality, the weak modality often exhibits lower predictive confidence. We investigate this phenomenon on CREMA-D \cite{CREMAD:journals/taffco/CaoCKGNV14}. As shown in Figure \ref{fig:intro} (left), we report accuracy, average top-1 confidence, and the average number of candidate classes with predicted probability $>1/C$, where $C$ is the number of classes. Higher confidence and fewer candidates indicate more certain predictions. Under Naive MML and OGM \cite{OGM:conf/cvpr/PengWD0H22}, substantial gaps exist between audio and video in both metrics. LFM \cite{LFM:conf/nips/0074WJ024} introduces a contrastive loss \cite{CLIP:conf/icml/RadfordKHRGASAM21} to intervene in label fitting, potentially driving modalities toward flatter suboptimal basins while substantially improving overall performance, demonstrating the effectiveness of cross-modal supervision. As shown in Figure \ref{fig:intro} (right), the confidence gap between strong and weak modalities appears on both correctly classified and misclassified samples, suggesting overconfidence in the strong modality and underconfidence in the weak one. In this paper, we analyze this problem through theory and experiments. We find that this discrepancy is not merely induced by multimodal learning, but fundamentally stems from unimodal characteristics. Specifically, for strong modality, the task gradient is positively correlated with the entropy reduction direction, causing predictions to become overly confident. In contrast, the weak modality is approximately orthogonal, resulting in smaller gains in predictive certainty. This asymmetric certainty gain already arises in unimodal learning and persists in multimodal learning. Ideally, each modality should produce confident but not overconfident predictions \cite{maxsup:conf/nips/abs-2502-15798,DBLP:conf/icml/WeiLL0SL22}, enabling effective multimodal fusion.

\begin{figure}[t]\centering
\includegraphics[width=.42\linewidth]{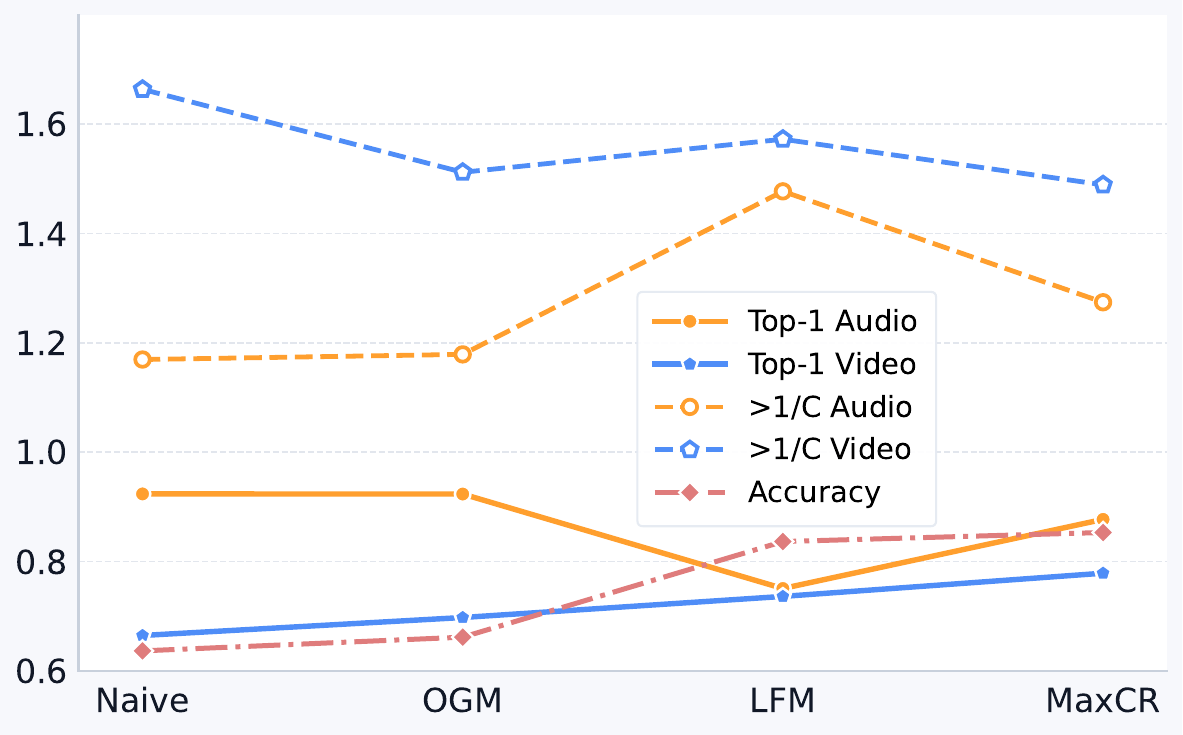} \hspace{0.04\linewidth}
\includegraphics[width=.42\linewidth]{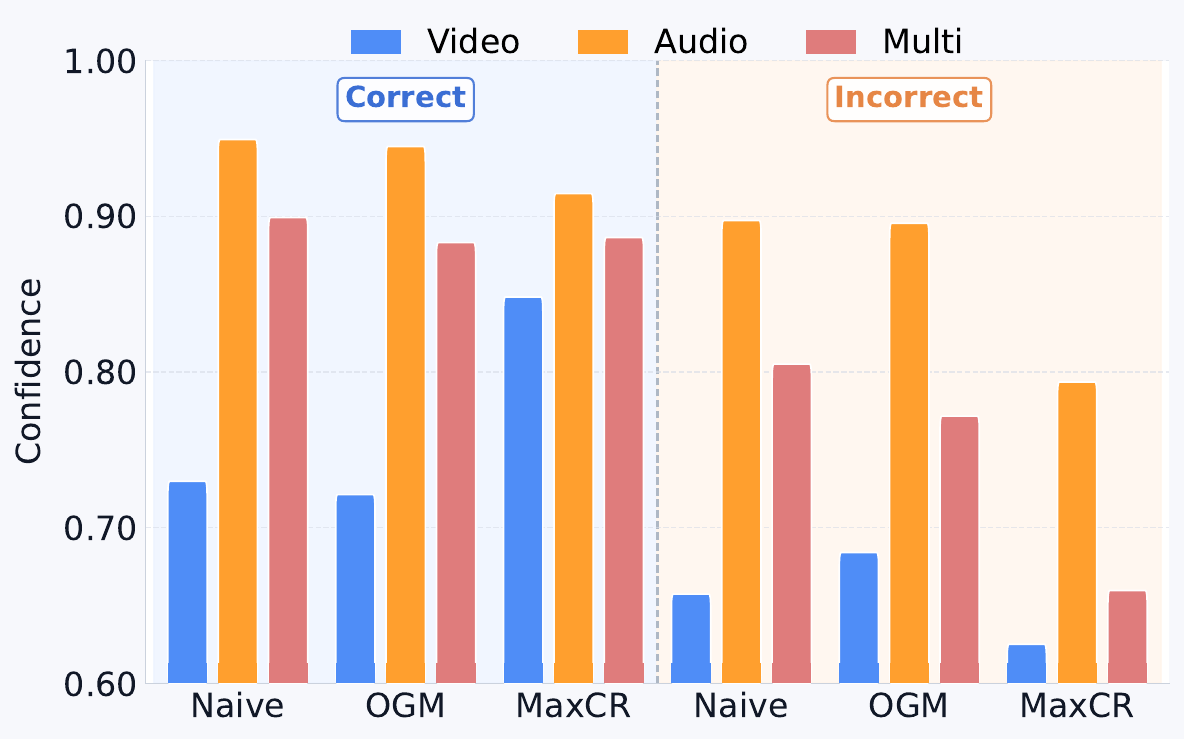} \\
\caption{Comparison of the semantic confidence on CREMAD dataset. Left: We observe that the strong modality (Audio) exhibits higher Top-1 confidence and fewer candidate probabilities $>1/C$, and the weak modality (Video) shows the opposite. Right: MaxCR mitigates underconfidence in correctly classified video samples and overconfidence in misclassified audio samples.}
\label{fig:intro}
\end{figure}

Based on the above insight, we propose a novel multimodal Max Confidence Regularization (MaxCR) approach that leverages cross-modal information to calibrate the unimodal confidence, thereby improving the performance of decision-level multimodal fusion. Concretely, we first adopt a nonlinear sparsity measure to dynamically monitor the semantic confidence of each modality during training. Then, we design max suppression and max excitation based on this monitoring metric to regularize strong and weak modalities, respectively. The former penalizes the top-1 confidence in strong modality, while the latter encourages it in weak ones, thereby balancing multimodal fusion. To this end, MaxCR can calibrate and balance semantic confidence between strong and weak modalities, thereby improving multimodal performance. Extensive experiments on widely used datasets reveal that our proposed method MaxCR achieves competitive performance compared with SOTA baselines.

\section{Related Work}
\subsection{Multimodal Learning}
Multimodal learning \cite{MMDL:conf/icml/NgiamKKNLN11,KMG:conf/ijcai/YangWZX019,TIL:journals/pami/BaltrusaitisAM19} seeks to leverage these multimodal data from diverse sensors to improve task performance. Generally, existing multimodal learning methods can be categorized into early fusion \cite{DOMFN:conf/mm/0074ZGGZ22,Film:conf/aaai/PerezSVDC18,ML-LSTM:journals/mta/NieYSW21}, late fusion \cite{KMG:conf/ijcai/YangWZX019,MSLR:conf/acl/YaoM22,MLA:conf/cvpr/ZhangYBY24}, and hybrid fusion \cite{Hybrid:conf/aaai/ZhengTWH023} based on their fusion strategies. In detail, early fusion methods typically perform joint modeling at the feature level, either by concatenating features to form a joint representation or by applying affine transformations. Representative early fusion methods include G-Blend \cite{OGR-GB:conf/cvpr/WangTF20}, DOMFN \cite{DOMFN:conf/mm/0074ZGGZ22}, and so on. In contrast, late fusion methods tend to model different modalities separately, independently extracting predictions from modality-specific models and integrating them at the decision stage, such as MSLR \cite{MSLR:conf/acl/YaoM22}. Additionally, hybrid fusion methods aim to combine the advantages of early and late fusion, typically introducing interaction mechanisms at different network layers to enable multi-level and multi-granularity cross-modal modeling \cite{Hybrid:conf/aaai/ZhengTWH023}. Despite numerous achievements, these methods may not fully exploit all modalities and produce under-optimized unimodal representations.

\subsection{Imbalanced Multimodal Learning}
Compared to unimodal approaches, multimodal learning is expected to achieve better performance \cite{MML:conf/cvpr/LvCHDL21,MML:conf/nips/SimonyanZ14,MML:conf/cvpr/GaoOGT20}. However, contrary to expectations, multimodal learning often encounters modality imbalance problems \cite{OGR-GB:conf/cvpr/WangTF20,OGM:conf/cvpr/PengWD0H22} in practice, leading to performance degeneration and even underperformance compared to unimodal ones in certain scenarios. Imbalanced multimodal learning aims to rebalance the learning of different modalities, ensuring that each modality is sufficiently exploited. Based on this idea, early pioneering works \cite{OGR-GB:conf/cvpr/WangTF20,OGM:conf/cvpr/PengWD0H22,PMR:conf/cvpr/Fan0WW023,AGM:conf/iccv/LiLHLLZ23} focus more on adaptively adjusting the optimization of different modalities. Representative methods, such as OGM \cite{OGM:conf/cvpr/PengWD0H22} and AGM \cite{AGM:conf/iccv/LiLHLLZ23}, etc., adopt different learning strategies to modulate gradients or learning rates, thereby rebalancing the learning of weak and strong modalities. Other approaches, including MLA \cite{MLA:conf/cvpr/ZhangYBY24}, ReconBoost \cite{ReconBoost:conf/icml/CongHua24}, and IGM \cite{IGM:conf/ijcai/JiangCY25}, take a different direction by focusing on decoupling multimodal learning and bridging the learning gap through injecting optimization information between modalities. In addition, several methods attempt to address the modality imbalance problem from different perspectives. Typically, LFM \cite{LFM:conf/nips/0074WJ024} considers modality imbalance from a label-fitting perspective. AUG \cite{AUG:conf/nips/Jiang2025aug} enhances the classification capability of weak modality to address modality imbalance.

\section{Problem Analysis} \label{sec:ProAna}
In this section, we analyze the reason for the semantic confidence disparity between strong and weak modalities, and attribute it to asymmetric certainty gains from optimization. To bridge analysis and practice, we further provide empirical validation, confirming the existence of the optimization-certainty decoupling, which underlies the asymmetry.

\subsection{Theoretical Analysis}
For weak modality $r$, we assume that its Bayes posterior is $\q^r(\x^r)=\PB(\Y\mid \X^r=\x^r)$, where $\x^r$ is a realization of the random variable $\X^r$. Let $\g_{\mathrm{ce}}=\nabla_{\theta}\LM_{\mathrm{CE}}$ denotes the population cross-entropy gradient, $\g_{e}=\nabla_\theta H$ denotes the entropy gradient, and $\hat \g=\frac{1}{N}\sum_{i=1}^{N}\J_i^{\top}(\p_i-\y_i)$ denote its mini-batch stochastic gradient, where $N$ is the mini-batch size, $\p_i$ is the predicted class-probability vector for the $i$-th sample, $\y_i$ is its one-hot label vector, and $\J_i$ is the Jacobian of the model logits with respect to parameters $\theta$. We further define the excess risk as $\EM(\theta)=\EB_\x\left[\mathrm{KL}(\q^{r}\|\p)\right]$, and introduce the certainty-transfer ratio $r_N(\theta)=\frac{|\langle \g_e,\g_{ce}\rangle|}{\|\g_e\|\sqrt{\EB\|\hat \g\|^2}}$, which measures the proportion of predictable population cross-entropy gradient signal along the certainty-related direction in the stochastic update. $\mathrm{KL}(\cdot\|\cdot)$ is KL divergence. Our analysis builds on a common view: weak modality is typically difficult to model, prone to ambiguity, and even contaminated by noise \cite{view:journals/tmm/MaiSXZH24,view:conf/mm/Ding0025,view:conf/nips/abs-2509-25831,view:conf/nips/gong2025multimoda}. More details and proofs can be found in the appendix.

\begin{theorem}[Optimization-Certainty Decoupling, Informal]
\label{thm:asymmetry}
Under the assumptions that the weak modality is ambiguous, i.e., $\EB_x[H(\q^{r})]\ge H_w>0$, where $H(\q^{r})$ is the entropy and $H_w$ is a constant representing the lower bound, and the model satisfies the regularity conditions. As $\EM(\theta)\to0$,
$\|\g_{\mathrm{ce}}\|\le B\sqrt{2\EM(\theta)}\to0$, $\EB\|\hat{\g}\|^2\ge \frac{\mu\phi_C(H_w)}{N}>0$. Consequently, $r_N(\theta) \le \frac{B\sqrt{2\EM(\theta)}} {\sqrt{2B^2\EM(\theta)+\mu\phi_C(H_w)/N}} \to0$, $\dot{\CM}(\theta)\to0$, where $B$, $N$ and $\mu$ are positive constants, and $\phi_C$ is a monotone function determined by the number of classes $C$.
\end{theorem}

Theorem \ref{thm:asymmetry} shows that, as a weak modality approaches its Bayes-optimal cross-entropy solution, the population optimization signal vanishes, while irreducible conditional label ambiguity can sustain non-negligible stochastic gradient fluctuations. Consequently, substantial mini-batch gradients do not necessarily correspond to effective learning, since their magnitude may increasingly reflect ambiguity-induced stochastic variation rather than coherent population-level optimization. Meanwhile, the instantaneous certainty gain also approaches zero, indicating that continued parameter updates become progressively less effective at improving predictive certainty. This decoupling can further manifest as weak directional alignment between task optimization and certainty improvement, which we empirically examine in the following section.

\begin{figure}[t]\centering
\includegraphics[width=.45\linewidth]{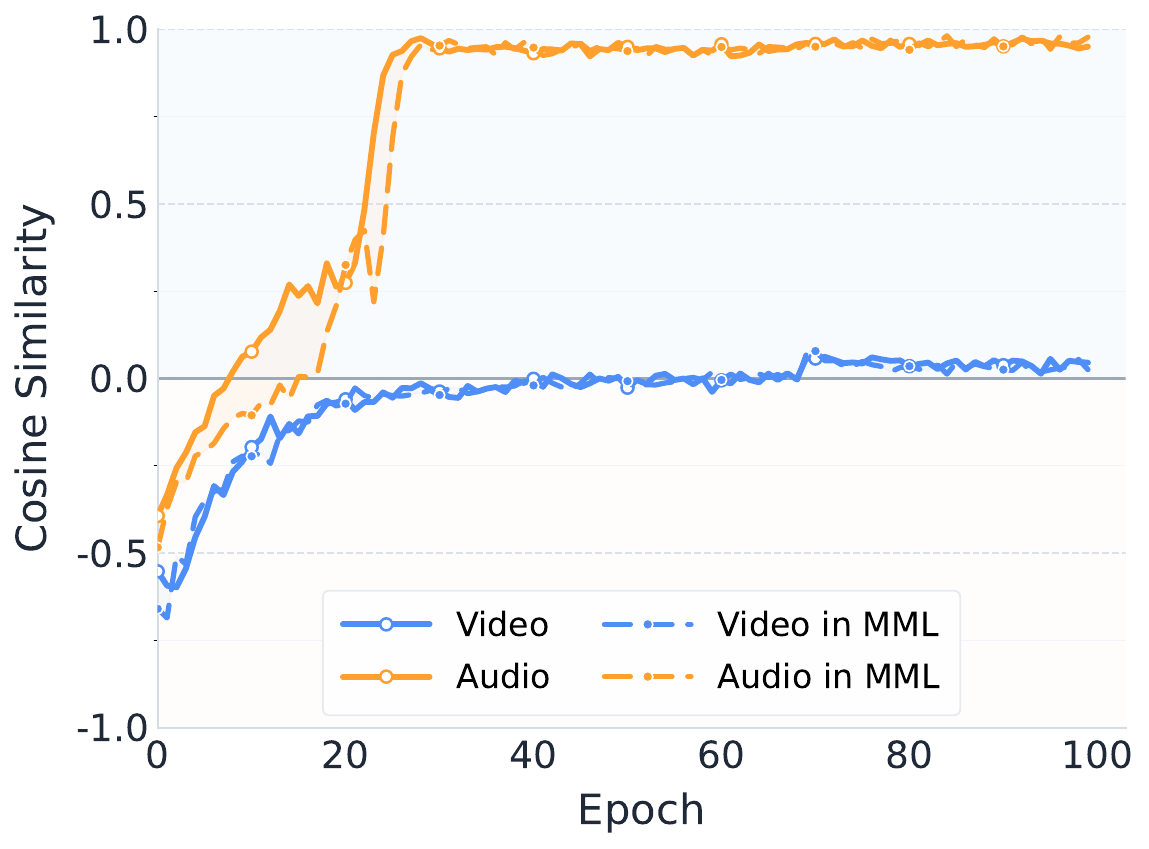} \hspace{0.04\linewidth}
\includegraphics[width=.45\linewidth]{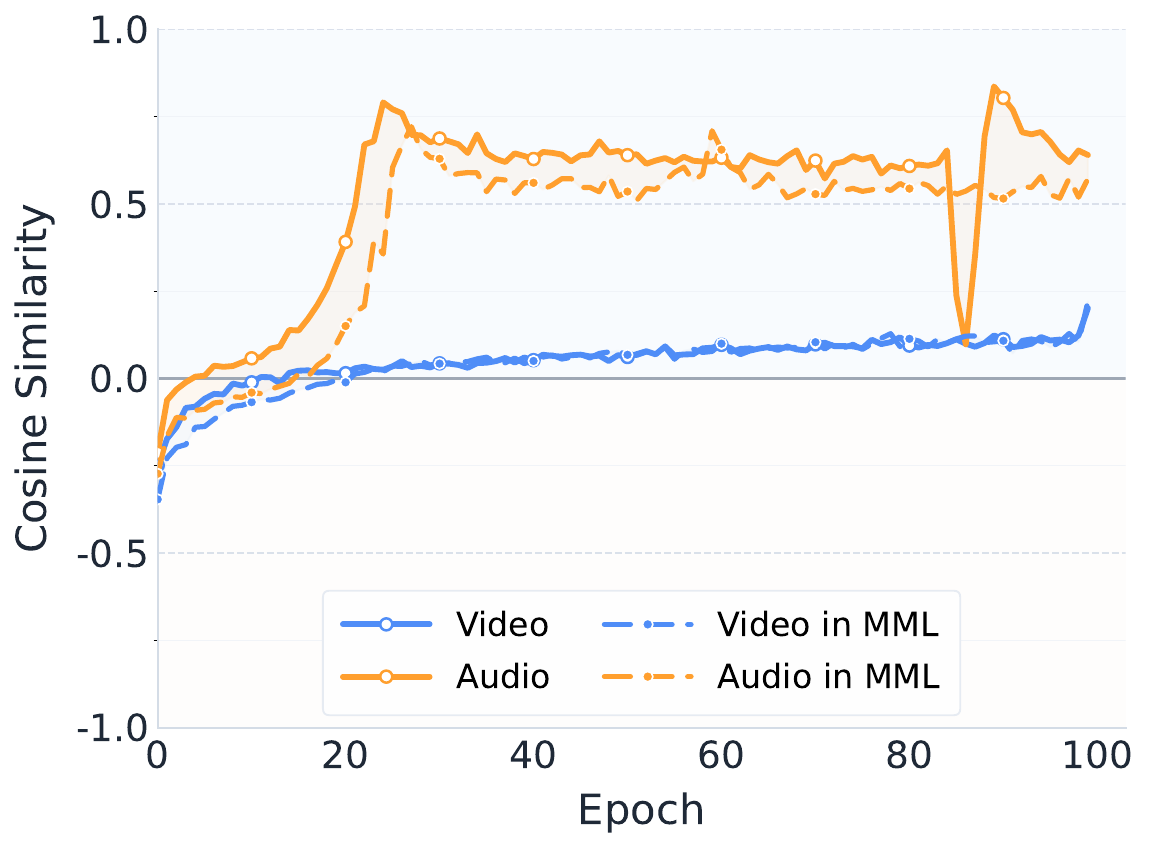} \\
\caption{Cosine similarity analysis between the task gradient and the entropy-reduced gradient on CREMAD (left) and KSounds (right) datasets. In unimodal learning (Audio and Video), the task gradient and the entropy-reduced gradient exhibit high similarity for audio modality, i.e., strong modality, whereas they are orthogonal for video modality, i.e., weak modality. In multimodal learning (Audio in MML and Video in MML), this phenomenon persists, leading to disparity in confidence.}
\label{fig:cosine_similarity}
\end{figure}

\subsection{Empirical Validation} \label{sec:Validation}
Our theoretical analysis attributes the semantic confidence gap between strong and weak modalities to optimization-certainty decoupling. However, this framework relies on some simplifying assumptions to ensure analytical tractability. To bridge the gap between our idealized model and real-world behavior, we conduct empirical investigations. We investigate the directional alignment between task optimization and certainty improvement as an empirical indicator of inefficient certainty transfer.

Concretely, we update the model using the cross-entropy loss and evaluate the cosine similarity between the gradients induced by the unimodal entropy loss and the cross-entropy loss during training. We conduct experiments on CREMAD \cite{CREMAD:journals/taffco/CaoCKGNV14} and KSounds \cite{Kinetics-Sound:conf/iccv/ArandjelovicZ17} datasets. As shown in Figure \ref{fig:cosine_similarity}, the results provide empirical evidence for a directional manifestation of optimization-certainty decoupling. We can observe that in unimodal learning, i.e., Audio and Video, the steepest descent direction of the audio modality is positively correlated with entropy reduction, whereas the cosine similarity of the video modality is close to zero, indicating that $r_N(\theta) \to0$. In multimodal learning, i.e., Audio in MML and Video in MML, this phenomenon also persists. This phenomenon indicates that the strong modality naturally improves its prediction confidence during training, and the weak modality struggles to translate task optimization gains into enhanced prediction certainty. Experimental results on other datasets are provided in the appendix.

\section{Methodology}
\subsection{Preliminaries}
For simplicity, we use two modalities, i.e., audio and video, for illustration. It is worth mentioning that our method can be easily adapted to cases with more than two modalities. Notation definitions and the trimodal details are provided in the appendix.

Assume that we have $M$ data points, each of which has audio and video modalities. Without loss of generality, we use $\DM=\{\x_i^a,\x_i^v, \y_i\}_{i=1}^M$ to denote a training dataset, where $\x_i^a$ and $\x_i^v$ denote the $i$-th data point of audio and video description, respectively. In addition, we are also given a category label $\y_i \in \{ 0,1\}^C$ for each data point, where $C$ denotes the number of category labels. With the training dataset $\DM$, the objective of multimodal learning is to learn a model that fuses the multimodal information and predicts its category label as accurately as possible.

\subsection{Multimodal Framework}
For the sake of simplicity, we use superscript $r$ to indicate the module corresponding to a specific modality, where $r\in\{a,v\}$. Following representative approaches \cite{OGR-GB:conf/cvpr/WangTF20,LFM:conf/nips/0074WJ024,AUG:conf/nips/Jiang2025aug}, we also utilize DNNs to construct our models. Specifically, we use $f^r(\cdot)$ to denote feature extractors, and the features can be calculated by $\h^r=f^r(\x^r;\theta^r)$, where $\theta^r$ denotes the extractor parameters. Then, the prediction of given data can be calculated by a classifier $\psi^r(\cdot)$: $\p^r=\psi^r(\h^r;\phi^r),$
where $\phi^r$ denotes the parameters of the classifier. Based on $\p^r$ and its ground-truth, the objective function can be written as:
\begin{align}
\LM_{CE} \left ( \x^r, \y \right )=\frac{1}{N}\sum_{i=1}^{N}\ell(\p^r_i,\y_i)=-\frac{1}{N}\sum_{i=1}^{N}\y_i^\top\log(\p^r_{i}),\label{obj:mml}
\end{align}
where $\ell(\cdot)$ denotes the cross-entropy loss. Based on this, we decouple multimodal learning to address optimization conflicts \cite{MLA:conf/cvpr/ZhangYBY24,ReconBoost:conf/icml/CongHua24}. We then introduce cross-modal information to regularize modality-specific prediction confidence, thereby injecting optimization information across modalities and narrowing the confidence gap between them.

\subsection{Max Confidence Regularization}
Based on the optimization-certainty decoupling phenomenon, we aim to address modality imbalance from the perspective of predictive confidence. Specifically, instead of directly reshaping optimization, we propose a novel approach, MaxCR, which dynamically identifies confidence discrepancy and regulates the maximum confidence of modality-specific predictions.

\noindent\textbf{Confidence Monitoring. }Since differences between modalities may evolve dynamically \cite{OGM:conf/cvpr/PengWD0H22}, an adaptive score is required to estimate the semantic confidence of each modality. Therefore, we introduce a nonlinear sparsity measure to monitor them, which considers the semantic confidence and the candidate classes. Given a prediction $\p^r$, the sparsity metric can be formally defined as:
\begin{align}
s^r_i=\frac{1}{L} \sum_{\p^r_{ij}\leq\epsilon}^{L} exp(\frac{-\p^r_{ij}}{\tau}),\label{obj:sparsity}
\end{align}
where $\tau \in \left ( 0,1 \right ]$ scales the sparsity metric, and $L$ denotes the number of entries smaller than $\epsilon$. $\epsilon$ is the threshold for entries, with $\epsilon=\frac{1}{C}$. Based on this metric, we compute the sparsity score at $t$-th step on the fly and apply mild smoothing to mitigate early-stage fluctuations: $S^r_t=\frac{t-1}{t} \frac{1}{N} \textstyle \sum_{i \in \B_t} s^r_i + \frac{1}{t} S^r_{t-1}$, where $\B_t$ is a mini-batch with size $N$ and $S_0^r=0$. A larger score indicates that the modality is more confident. Hence, if $S^a_t-S^v_t \geq \sigma$, we suppress the audio modality and excite the video modality, where $\sigma$ is the dead zone for fault tolerance. Conversely, if $S^v_t-S^a_t \geq \sigma$, the video modality is suppressed and the audio modality is excited. We then design the discrepancy ratio $\lambda$:
\begin{align}
\lambda=\left | S^a_t - S^v_t\right |.\label{obj:ratio}
\end{align}

This sparsity metric is more sensitive than entropy when predictions are uncertain and close to uniform, whereas entropy is more sensitive when the probability mass is concentrated on a few classes. Since label smoothing and entropy loss may over-amplify misclassified confidence \cite{maxsup:conf/nips/abs-2502-15798,Entropy:conf/iclr/ZhangBKZZ25}, we adopt a conservative peak regularization, ensuring the combined objective does not reverse the cross-entropy correction for an incorrect top-1 prediction. More details on the effectiveness and conservativeness are provided in the appendix.

\noindent\textbf{Max Suppression. }For over-confident modalities, we reduce excessive certainty. We formulate max suppression, which explicitly penalizes the top-1 confidence. Specifically, given a prediction $\p^r_i$, we compute the maximum probability and compare it with the mean. The loss is defined as follows:
\begin{align}
\LM_{sup}(\x^r, \y)=\frac{1}{N}\sum_{i=1}^{N}\lambda(\max_j \p^r_{ij} - \frac{1}{C}).\label{obj:sup}
\end{align}
This loss reduces the top-1 confidence during optimization and encourages smooth predictions. In this way, max suppression mitigates the overconfidence of strong modality, thereby improving generalization. By combining (\ref{obj:mml}) and (\ref{obj:sup}), the strong modality objective can be defined as:
\begin{align}
\LM_{s}(\x^r, \y)=\LM_{CE}( \x^r, \y )+ \LM_{sup}(\x^r, \y).\label{obj:strong}
\end{align}

\noindent\textbf{Max Excitation. }For under-confident modalities, we encourage sharper predictions. We propose max excitation, which explicitly encourages the top-1 confidence. Similarly, given the prediction $\p^r_i$, we encourage maximum prediction to separate from the mean. The max excitation loss is defined as:
\begin{align}
\LM_{exc}(\x^r, \y)=\frac{1}{N}\sum_{i=1}^{N}\lambda(\frac{1}{C} - \max_j \p^r_{ij}).\label{obj:exc}
\end{align}
This mechanism corrects predictions under high ambiguity, thereby alleviating underconfidence in weak modality. By combining (\ref{obj:mml}) and (\ref{obj:exc}), the weak modality objective can be defined as:
\begin{align}
\LM_{w}(\x^r, \y)=\LM_{CE}( \x^r, \y )+ \LM_{exc}(\x^r, \y).\label{obj:weak}
\end{align}

MaxCR is summarized in Algorithm \ref{algo:ours}. We obtain multimodal prediction by averaging the predictions across all modalities at inference. Combined with MaxCR, multimodal learning can be rebalanced. 

\begin{algorithm}[t]
\caption{Learning algorithm of MaxCR.}\label{algo:ours}
\begin{algorithmic}[1]
\Require{Training dataset $\DM=\{\x_i^a,\x_i^v, \y_i\}_{i=1}^M$.}
\Ensure{The learned DNN models for all modalities.}\\
\textbf{INIT} Initialize DNN parameters $\{\theta^r,\phi^r\}$,$r\in\{a,v\}$. Initialize iteration $t=1$.
\For{$t=1\mapsto\#iterations$}
\State Sample a mini-batch $\B_t=\{\x_i^a,\x_i^v,\y_i\}_{i=1}^{N}$;
\State $\forall \x_i^a,\x_i^v\in\B_t$, calculate predictions $\p^a_{i}$ and $\p^v_{i}$;
\State Calculate sparsity score $\{S^a_t,S^v_t\}$ based on predictions;
\Comment{{\color{blue}Max Confidence Regularization.}}
\State Calculate discrepancy ratio $\lambda$ in (\ref{obj:ratio});
\If{$S^a_t-S^v_t \geq \sigma$} 
\Comment{{\color{blue}Max Suppression and Max Excitation.}}
\State Calculate loss $\LM_{s}(\x^a, \y)$ and $\LM_{w}(\x^v, \y)$ in (\ref{obj:strong}) and (\ref{obj:weak});
\ElsIf{$S^v_t-S^a_t > \sigma$}
\State Calculate loss $\LM_{s}(\x^v, \y)$ and $\LM_{w}(\x^a, \y)$ in (\ref{obj:strong}) and (\ref{obj:weak});
\Else
\State Calculate loss in (\ref{obj:mml});
\EndIf
\State Update DNN parameters $\{\theta^r,\phi^r\}$,$r\in\{a,v\}$;
\EndFor
\end{algorithmic}
\end{algorithm}

\begin{table*}[t]
\centering
\caption{Comparison with SOTA multimodal learning methods. The best results are highlighted in
{\bf bold}. The {\underline{underline}} denotes the second-best performance. The results with gray background denote the performance based on multimodal learning is inferior to that of the best unimodal approach.}
\label{tab:main_results}
\renewcommand\arraystretch{1.3}
\resizebox{\textwidth}{!}{
\begin{tabular}{l|cc|cc|cc|cc|cc|cc}
\toprule
\multirow{2}{*}{\textbf{Method}} &\multicolumn{2}{c|}{CREMAD} &\multicolumn{2}{c|}{KSounds} &\multicolumn{2}{c|}{VGGSound} &\multicolumn{2}{c|}{Twitter} &\multicolumn{2}{c|}{Sarcasm}&\multicolumn{2}{c}{NVGesture}\\
\cmidrule(lr){2-3} \cmidrule(lr){4-5} \cmidrule(lr){6-7} \cmidrule(lr){8-9} \cmidrule(lr){10-11}\cmidrule(lr){12-13}
& Acc. & MAP    & Acc. & MAP    & Acc. & MAP    & Acc. & F1    & Acc. & F1    & Acc. & F1    \\
\midrule
Unimodal-1     & .6317  & .5879 & .5412  & .5669 & .4655 & .4701 & .5863 & .4333 & .7181 & .7073 & .7822 & .7833 \\
Unimodal-2     & .4583  & .6861 & .5562  & .5837 & .3494 & .3478 & .7367 & .6849 & .8136 & .8056 & .7863 & .7865 \\
Unimodal-3     & $-$    & $-$   & $-$    & $-$   & $-$   & $-$   & $-$   & $-$   & $-$   & $-$   & .8154 & .8183 \\\midrule
Concat         & .6361  & \unimodal{.6841} & .6455  & .7130 & .5116 & .5352 & \unimodal{.7011} & \unimodal{.6386} & .8286 & .8240 & .8237 & .8270 \\
Affine         & .6626  & .7193 & .6424  & .6931 & .5001 & .5155 & \unimodal{.7203} & \unimodal{.5992} & .8240 & .8188 & .8278 & .8281 \\
ML-LSTM        & \unimodal{.6290}  & \unimodal{.6473} & .6394  & .6902 & .4966 & .5139 & \unimodal{.7068} & \unimodal{.6564} & .8277 & .8205 & .8320 & .8330 \\\midrule
G-Blend        & .6465  & .7392 & .6722  & .7274 & .5086 & .5555 & \unimodal{.7309} & \unimodal{.6799} & .8286 & .8215 & .8299 & .8305 \\
MSLR           & .6868  & .7412 & .6756  & .7282 & .4987 & .5415 & \unimodal{.7232} & \unimodal{.6382} & .8439 & .8378 & .8237 & .8284 \\
OGM            & .6612  & .7372 & .6582  & .7159 & .4829 & .4978 & \unimodal{.7058} & \unimodal{.6435} & .8360 & .8293 & $-$   & $-$   \\
PMR            & .6659  & .7058 & .6675  & .7274 & \unimodal{.4647} & .4866 & \unimodal{.7357} & \unimodal{.6636} & .8310 & .8256 & $-$   & $-$   \\
AGM            & .6733  & .7807 & .6791  & .7388 & .4711 & .5198 & \unimodal{.7261} & \unimodal{.6502} & .8306 & .8293 & .8279 & .8284 \\
MMPareto       & .7487  & .8535 & .7000  & .7850 & .5125 & .5473 & \unimodal{.7358} & \unimodal{.6729} & .8348 & .8284 & .8382 & .8424 \\
SMV            & .7872  & .8417 & .6900  & .7426 & .5031 & .5362 & .7428 & \unimodal{.6817} & .8418 & .8368 & .8352 & .8341 \\
MLA            & .7943  & .8572 & .7004  & \underline{.7945} & .5165 & .5473 & \unimodal{.7352} & \unimodal{.6713} & .8426 & .8348 & .8340 & .8372 \\
DI-MML         & .8158  & .8592 & .7203  & .7426 & .5173 & .5479 & \unimodal{.7248} & \unimodal{.6686} & .8411 & .8315 & $-$   & $-$   \\
ReconBoost     & .7557  & .8140 & .6855  & .7662 & .5097 & .5387 & .7442 & \unimodal{.6832} & .8437 & .8317 & .8386 & .8434\\
LFM            & .8362  & .9006 & .7253  & .7897 & .5274 & .5694 & .7501 & \bf .7057 & .8497 & .8457 & .8436 & .8468 \\
AMSS+          & .7030  & .7641 & .7255  & .7913 & $-$ & $-$ & \bf .7558 & \underline{.6981} & .8435 & .8377 & .8464 & .8494   \\
InfoReg        & .7498  & .8603 & .7189  & .7986 & .5289 & .5712 & .7403 & .6847 & .8479 & .8396 & $-$ & $-$   \\
AUG            & \underline{.8515}  & \underline{.9103} & \underline{.7263}  & .7901 & \underline{.5301} & \bf .5826 & {.7512} & .6962 & \underline{.8510} & \underline{ .8458} & \underline{.8501} & \underline{.8533}   \\\midrule
MaxCR          & \makecell{\bf .8531 \\ \std{0.0044}}  & \makecell{\bf .9216 \\ \std{0.0046}} & \makecell{\bf .7522 \\ \std{0.0035}}  & \makecell{\bf .8102 \\ \std{0.0067}} & \makecell{\bf .5416 \\ \std{0.0011}} & \makecell{\underline{.5794} \\ \std{0.0032}} & \makecell{\underline{.7517} \\ \std{0.0028}} & \makecell{.6934 \\ \std{0.0043}} & \makecell{\bf .8518 \\ \std{0.0021}} & \makecell{\bf .8459 \\ \std{0.0016}} & \makecell{\bf .8651 \\ \std{0.0062}}   & \makecell{\bf .8680 \\ \std{0.0043}}   \\
\bottomrule
\end{tabular}
}
\end{table*}

\section{Experiments}
\subsection{Experimental Setup} \label{exp:setup}
{\bf Dataset:} \textbf{CREMAD} \cite{CREMAD:journals/taffco/CaoCKGNV14} is an audio-visual dataset for emotion recognition, containing 7,442 video clips from 91 actors speaking several short words. \textbf{KSounds} \cite{Kinetics-Sound:conf/iccv/ArandjelovicZ17} is designed for human action recognition, selected from Kinetics dataset, and contains 31 action classes. \textbf{VGGSound} \cite{VGGSound:conf/icassp/ChenXVZ20} is a large-scale audio-visual dataset in the wild with 309 classes and nearly 200k 10-second video clips. \textbf{Twitter} \cite{Twitter15:conf/ijcai/Yu019} is an image-text dataset designed for emotion recognition, consisting of tweets collected from Twitter with three classes. \textbf{Sarcasm} \cite{Sarcasm:conf/acl/CaiCW19} dataset is designed for sarcasm detection, consisting of 24,635 image-text pairs. \textbf{NVGesture} \cite{NVGeasture:conf/cvpr/MolchanovYGKTK16} is a trimodal dataset for gesture recognition, containing 1,532 dynamic gestures. More details are provided in the appendix.

{\bf Baselines:} We select two categories of methods for comparison, i.e., traditional multimodal fusion methods and rebalanced multimodal learning methods. In detail, traditional multimodal fusion methods include concatenation fusion (Concat), affine transformation fusion (Affine) \cite{Film:conf/aaai/PerezSVDC18}, and multi-layer LSTM fusion (ML-LSTM) \cite{ML-LSTM:journals/mta/NieYSW21}. Rebalanced multimodal learning methods include G-Blend \cite{OGR-GB:conf/cvpr/WangTF20}, MSLR \cite{MSLR:conf/acl/YaoM22}, PMR \cite{PMR:conf/cvpr/Fan0WW023}, AGM \cite{AGM:conf/iccv/LiLHLLZ23}, MMPareto \cite{MMPareto:conf/icml/WeiH24}, SMV \cite{SMV:conf/cvpr/YakeRZD24}, MLA \cite{MLA:conf/cvpr/ZhangYBY24}, DI-MML \cite{DI-MML:conf/mm/FanXWLG24}, ReconBoost \cite{ReconBoost:conf/icml/CongHua24}, LFM \cite{LFM:conf/nips/0074WJ024}, AMSS+ \cite{AMSS:yang2025learning}, InfoReg \cite{InfoReg:huang2025adaptive}, AUG \cite{AUG:conf/nips/Jiang2025aug}.

{\bf Evaluation Protocols:} We adopt accuracy, mean average precision (MAP), and MacroF1 as evaluation metrics. The accuracy measures the proportion of correct predictions of total predictions. MAP returns the average precision of all samples. MacroF1 calculates the average F1 across all categories.

{\bf Implementation Details: }Following \cite{OGM:conf/cvpr/PengWD0H22}, we employ ResNet18 \cite{ResNet:conf/cvpr/HeZRS16} as the backbone to encode audio and video for CREMAD, KSounds and VGGSound datasets. All the parameters of the backbone are randomly initialized. For NVGesture dataset, we employ the I3D \cite{I3D:conf/cvpr/CarreiraZ17} as unimodal branch following the setting of \cite{nvGesture:conf/icml/WuJCG22}. We initialize the encoder with the pre-trained model trained on ImageNet. We use SGD with a momentum of 0.9 and a weight decay of 1e-4 as the optimizer. The learning rate is initialized at 1e-2 and decayed by a factor of 10 when the training loss reaches saturation for CREMAD, KSounds, VGGSound, and NVGesture datasets. For Twitter and Sarcasm datasets, following \cite{Twitter15:conf/ijcai/Yu019,Sarcasm:conf/acl/CaiCW19}, we adopt BERT \cite{BERT:conf/naacl/DevlinCLT19} as the text encoder and ResNet50 \cite{ResNet:conf/cvpr/HeZRS16} as the image encoder. We use Adam \cite{Adam:journals/corr/KingmaB14} as the optimizer, with an initial learning rate of 2e-5. For all datasets, $\tau$ is selected validation set from $\{\frac{1}{C}, \frac{1}{2C}, \frac{1}{3C}, \frac{1}{4C}, \frac{1}{5C}\}$ with the default set to $\frac{1}{C}$, while $\sigma$ is set to be 0.1, where $C$ is the number of categories. For comparison methods, the source codes of all baselines are kindly provided by their authors, and all baselines adopt the same backbone and initialization strategy. We conduct experiments with three random seeds. All experiments are conducted on NVIDIA GeForce RTX 4090, and all models are implemented with PyTorch.

\subsection{Main Results}

{\bf Comparison with Multimodal Baseline:} The main results on all datasets are reported in Table \ref{tab:main_results}. Unimodal-1/2 respectively denote the audio/video for CREMAD and KSounds, and image/text for Twitter and Sarcasm. Unimodal-1/2/3 denotes the RGB/OF/Depth modality for NVGesture dataset, respectively. Furthermore, the results with a gray background indicate that the performance based on multimodal learning is inferior to that of the best unimodal approach. First, the results show that our method consistently achieves competitive performance on all datasets. In contrast, compared with direct unimodal learning, some multimodal baseline methods fail to yield improvements and even produce worse results. Moreover, results on NVGesture dataset demonstrate that MaxCR effectively addresses challenges in scenarios involving more than two modalities and achieves best performance.

\begin{figure*}[t] 
\begin{minipage}[t]{0.49\textwidth}\centering
\includegraphics[width=.49\linewidth]{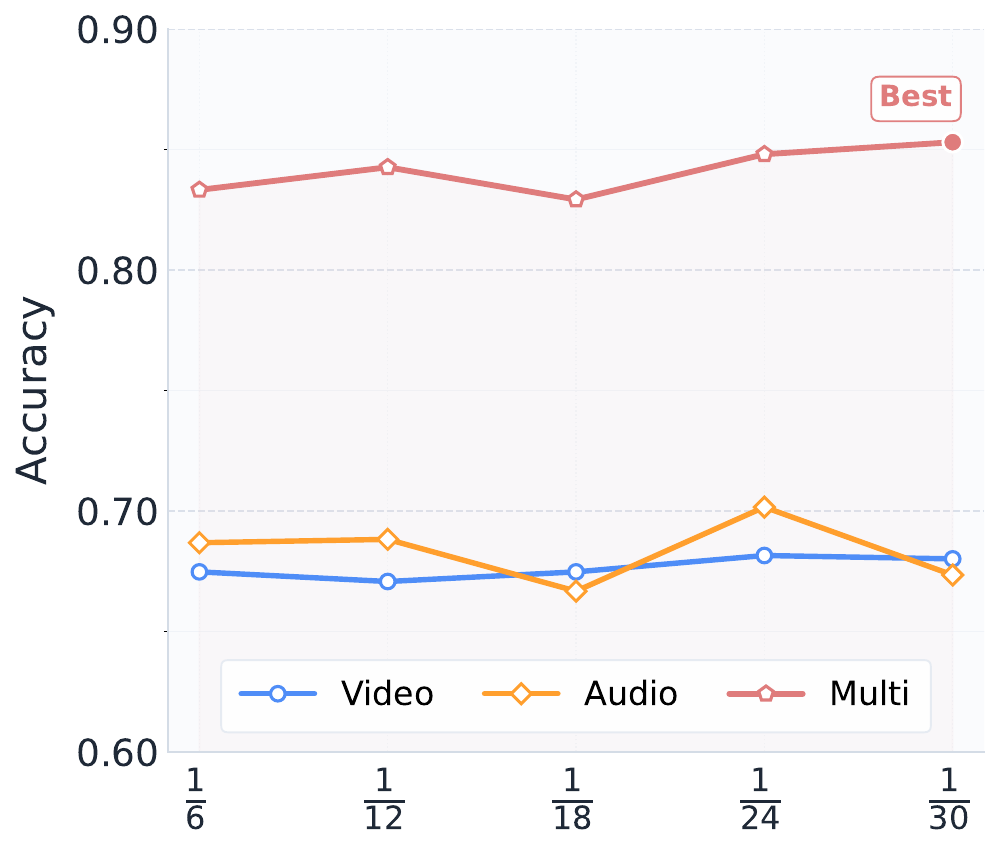}
\includegraphics[width=.49\linewidth]{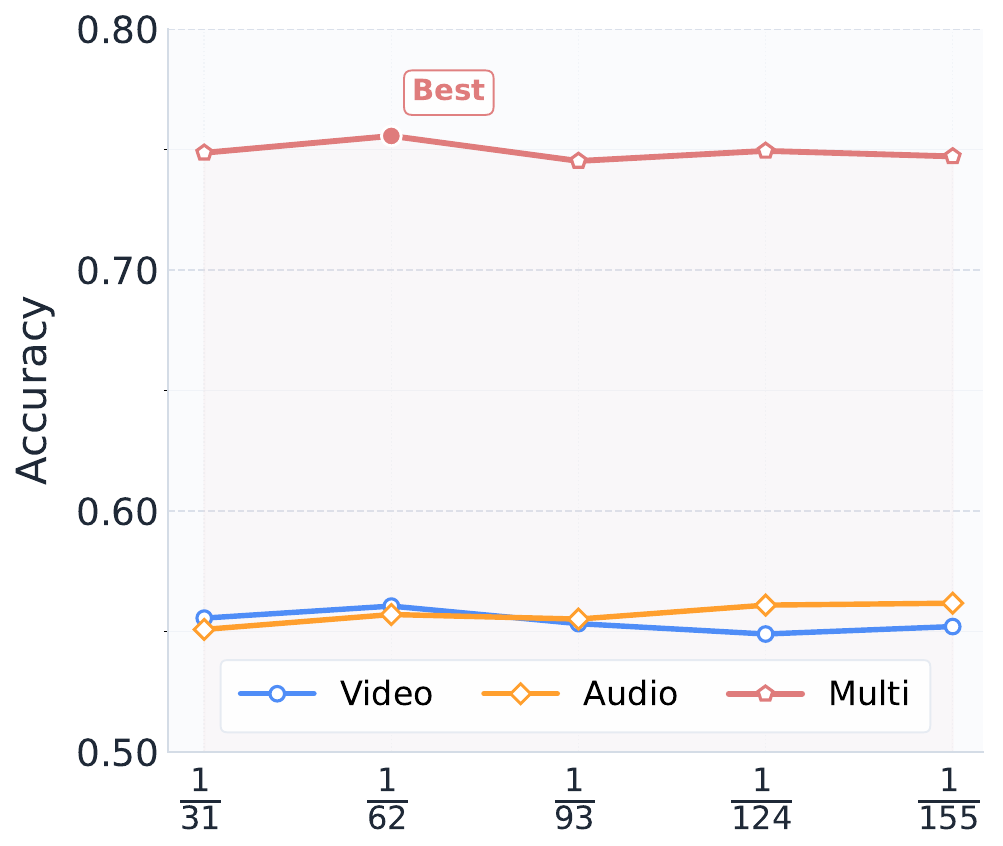} \\
\caption{Sensitivity analysis about $\tau$ on CREMAD (left) and KSounds (right) datasets.}
\label{fig:sensitivity}
\end{minipage}
\hfill
\begin{minipage}[t]{0.49\textwidth}\centering
\includegraphics[width=.49\linewidth]{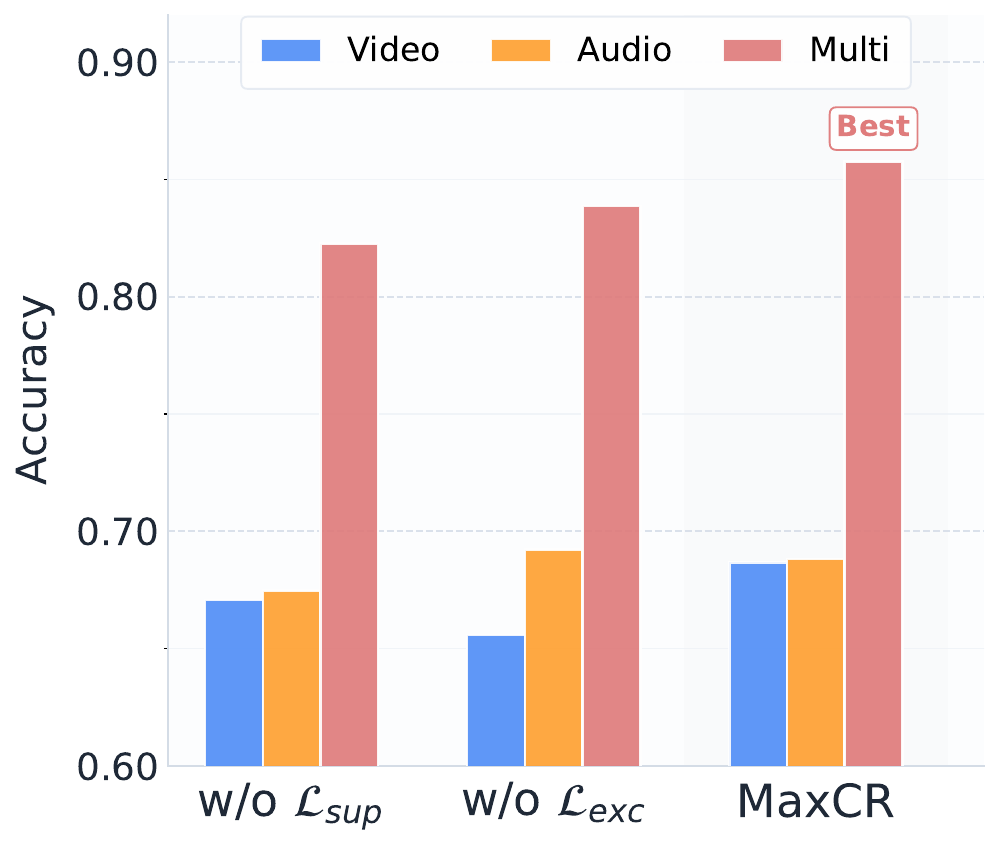}
\includegraphics[width=.49\linewidth]{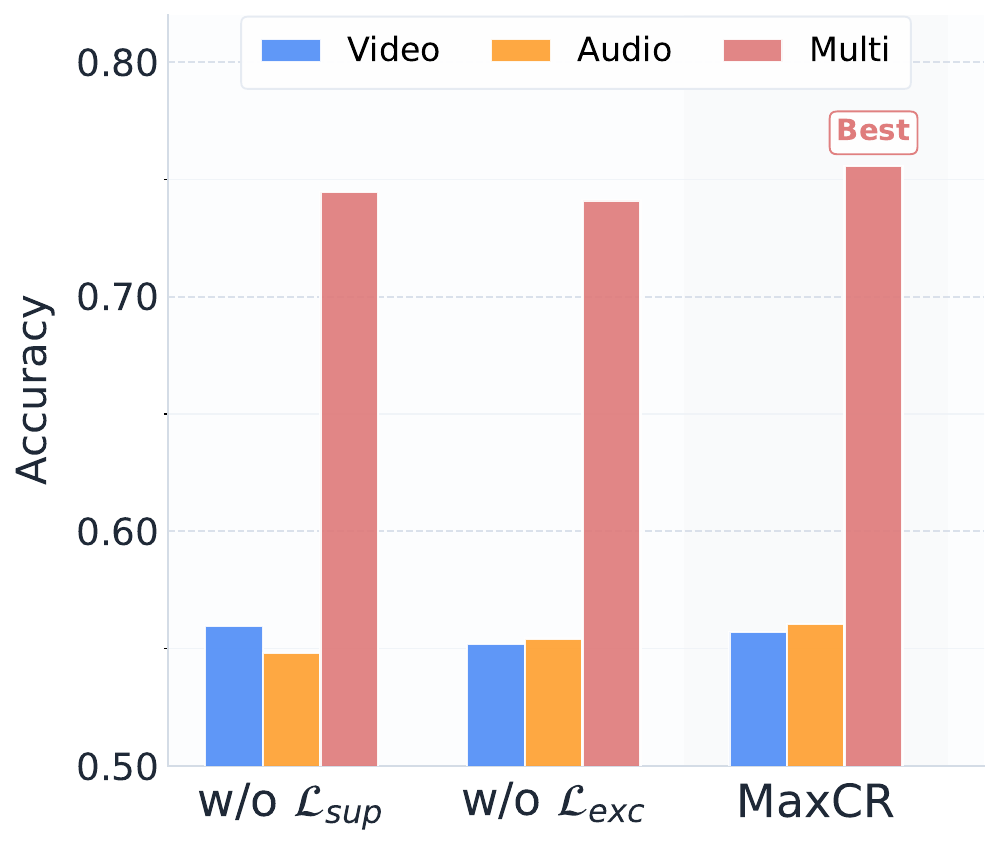} \\
\caption{Ablation study of MaxCR on CREMAD (left) and KSounds (right) datasets.}
\label{fig:ablation}
\end{minipage}
\end{figure*}

\begin{table}[t]
\centering
\begin{tabular}{cc}
\begin{minipage}{.45\textwidth}
\caption{Performance with different calibration metrics on KSounds dataset.}
\label{tab:ece}
\centering
\renewcommand\arraystretch{1.3}
\resizebox{\textwidth}{!}{
\begin{tabular}{c|cccc}
\toprule
Method    & Acc   & ECE  & NLL & Brier  \\\midrule
Naive     & 0.6455  & 0.1607  & 1.5206 & 0.5167     \\
OGM       & 0.6582   & 0.1498  & 1.4561 & 0.5009      \\
MaxCR     & \bf 0.7522  & \bf 0.0384  & \bf 0.9195 & \bf 0.3466       \\
\bottomrule
\end{tabular}
}
\end{minipage}
&
\begin{minipage}{.45\textwidth}
\caption{Performance with different monitoring metrics on KSounds dataset.}
\label{tab:score}
\centering
\renewcommand\arraystretch{1.3}
\resizebox{\textwidth}{!}{
\begin{tabular}{c|ccc}
\toprule
Method     &  Audio & Video & Multimodal  \\\midrule
Confidence      & 0.5322 & 0.5354 & 0.7317 \\
Entropy     & 0.5385 & 0.5443 & 0.7352 \\
Sparsity (Our)      & \bf 0.5570 & \bf 0.5605 & \bf 0.7522 \\
\bottomrule
\end{tabular}
}
\end{minipage} 
\end{tabular}
\end{table}

\subsection{Sensitivity and Ablation Study}
{\bf Sensitivity to the scaling factor $\tau$: }We explore the influence of scaling factor $\tau$ on CREMAD and KSounds datasets. We present the accuracy with different $\tau \in \left [ \frac{1}{5C}, \frac{1}{C}\right ]$ in Figure \ref{fig:sensitivity}. Although the optimal $\tau$ varies across datasets, MaxCR remains insensitive to the scaling factor $\tau$ in a large range.

{\bf Ablation Study:} We analyzed the impact of losses $\LM_{sup}$ and $\LM_{exc}$, i.e., Equation (\ref{obj:sup}) and (\ref{obj:exc}), on CREMAD and KSounds datasets to investigate the effectiveness of MaxCR. As shown in Figure \ref{fig:ablation}, we can observe that when the max suppression loss is removed (w/o $\LM_{sup}$), the accuracy of the strong modality, i.e., Audio, decreases. Conversely, when the max excitation loss is removed (w/o $\LM_{exc}$), the accuracy of the weak modality, i.e., Video, declines. When both max suppression and max excitation are used together, the two modalities achieve balance, leading to the best multimodal performance.

\subsection{Further Analysis}
 {\bf Calibration Metrics Evaluation:} We incorporate the Expected Calibration Error (ECE), Negative Log-Likelihood (NLL), and Brier Score (Brier) metrics to evaluate confidence. Fundamentally, ECE quantifies the discrepancy between predictive confidence and empirical accuracy. NLL measures the probability assigned to the ground-truth label, penalizing confident mispredictions, and the Brier Score captures the squared deviation between predicted probabilities and the true label distribution. We conduct experiments on KSounds dataset. As reported in Table \ref{tab:ece}, the results compellingly demonstrate that, beyond performance gains, MaxCR yields more reliable semantic confidence.

{\bf Selection of Monitoring Metric:} We introduce a sparsity metric to quantify the confidence discrepancy across modalities. To some extent, entropy and label confidence can also serve as monitoring metrics. To evaluate the effectiveness of entropy and confidence, we conduct experiments on KSounds dataset. Specifically, we use normalized entropy and the average label confidence to ensure the same numerical scale. As shown in Table \ref{tab:score}, the entropy achieves better performance than confidence, while the sparsity metric achieves the best accuracy. More results are provided in the appendix.

\begin{figure*}[t] 
\begin{minipage}[t]{0.33\linewidth}\centering
\includegraphics[width=\linewidth]{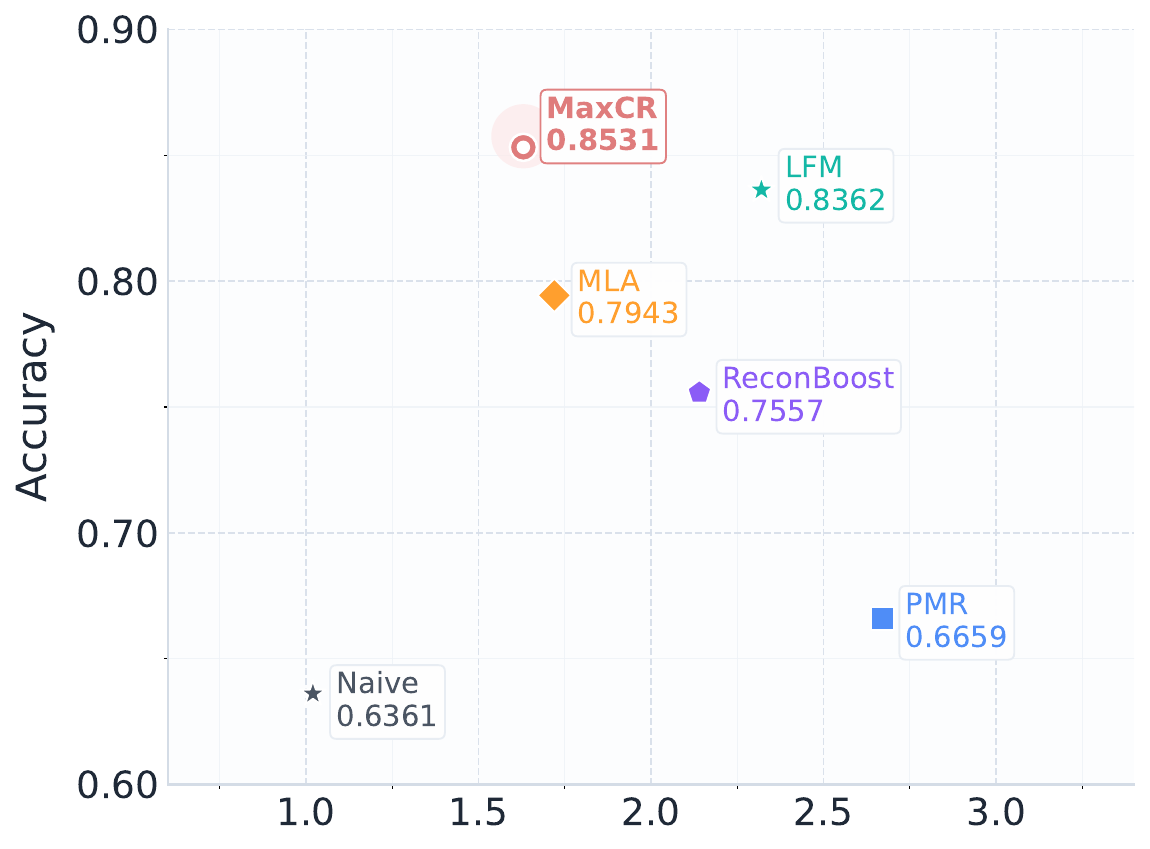} \\
\caption{Training time (hrs).}
\label{fig:train-time}
\end{minipage}\hfill
\begin{minipage}[t]{0.33\linewidth}\centering
\includegraphics[width=\linewidth]{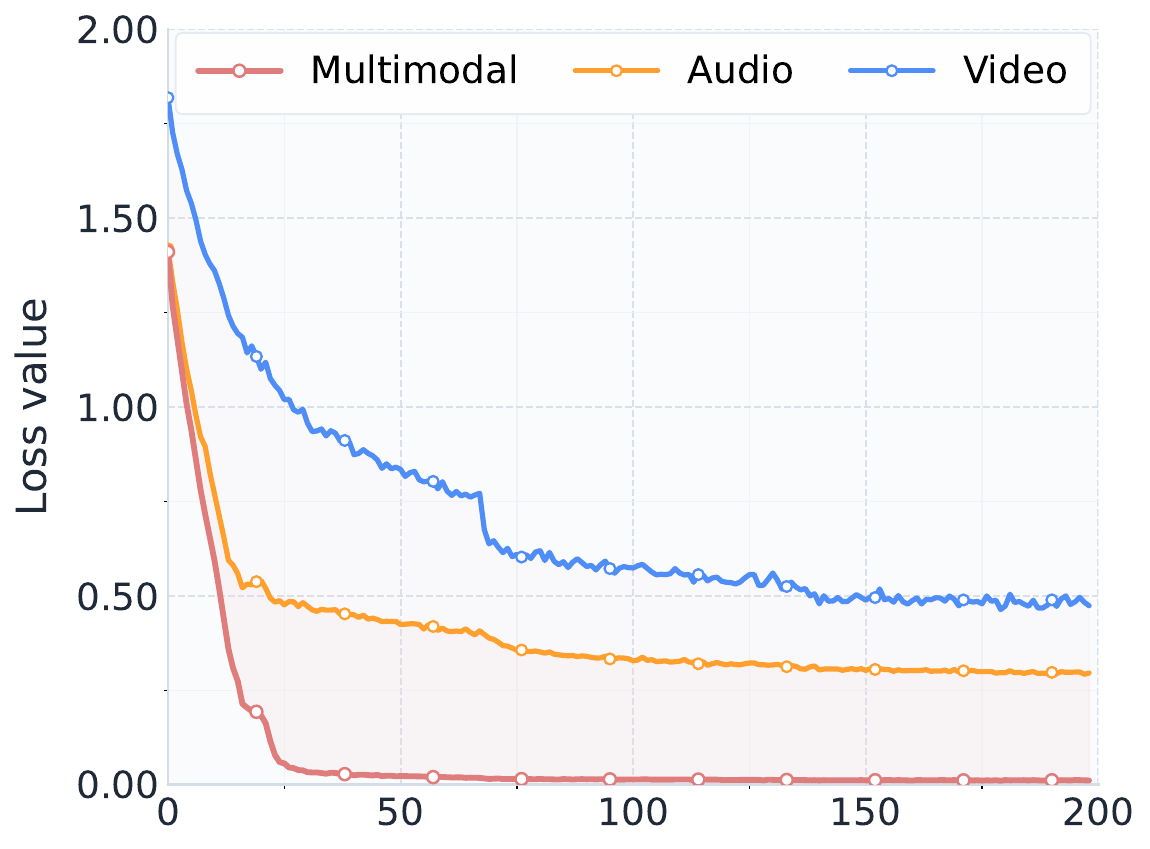} \\
\caption{Loss visualization.}
\label{fig:loss}
\end{minipage}\hfill
\begin{minipage}[t]{0.33\linewidth}\centering
\includegraphics[width=\linewidth]{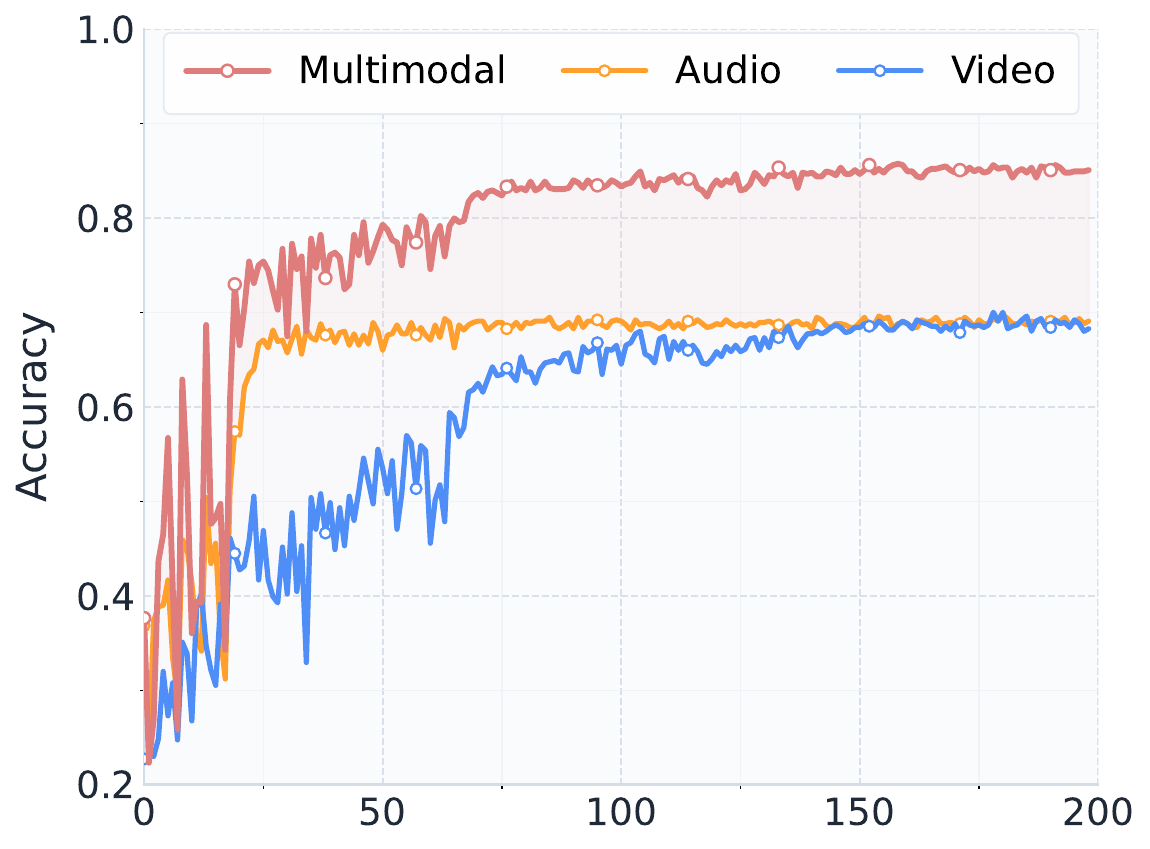} \\
\caption{Accuracy visualization.}
\label{fig:accuracy}
\end{minipage}
\end{figure*}

{\bf Max Suppression vs. Label Smoothing:} Since max suppression loss is designed to mitigate overconfidence, it naturally relates to label smoothing, which combines a uniform distribution with hard ground-truth labels to alleviate overconfidence. We experiment on KSounds dataset to examine whether label smoothing yields gains. Specifically, label smoothing is dynamically applied to the strong modality during training. The results in Table \ref{tab:ls} show that the improvement brought by label smoothing (w/ LS) is weaker than that of the proposed max suppression (w/ Sup) strategy. Max suppression improves the performance of the audio modality and enhances its collaboration with the video modality. This may be because max suppression regularizes the peaks of model predictions, rather than uniformly suppressing the ground-truth class.

{\bf Max Excitation vs. Entropy Minimization:} Since max excitation loss encourages more confident predictions, its objective is closely related to that of entropy minimization. Therefore, we further compare the max excitation loss with the entropy minimization loss on KSounds dataset. Specifically, we apply entropy minimization loss to the weak modality. As shown in Table \ref{tab:em}, the gain from entropy minimization (w/ EM) is smaller than that from max excitation (w/ Exc).

\begin{table}[t]
\centering
\begin{tabular}{cc}
\begin{minipage}{.45\textwidth}
\caption{Comparison of label smoothing and max suppression on KSounds dataset.}
\label{tab:ls}
\centering
\renewcommand\arraystretch{1.3}
\resizebox{\textwidth}{!}{
\begin{tabular}{c|ccc}
\toprule
Method     &  Audio & Video & Multimodal  \\\midrule
w/ LS      & 0.5466 & 0.5226 & 0.7286 \\
w/ Sup     & 0.5543 & 0.5520 & 0.7410 \\
MaxCR      & \bf 0.5570 & \bf 0.5605 & \bf 0.7522 \\
\bottomrule
\end{tabular}
}
\end{minipage}
&
\begin{minipage}{.45\textwidth}
\caption{Comparison of entropy minimization and max excitation on KSounds dataset.}
\label{tab:em}
\centering
\renewcommand\arraystretch{1.3}
\resizebox{\textwidth}{!}{
\begin{tabular}{c|ccc}
\toprule
Method     &  Audio & Video & Multimodal  \\\midrule
w/ EM      & 0.5474 & 0.5280 & 0.7248 \\
w/ Exc     & 0.5481 & 0.5597 & 0.7445 \\
MaxCR      & \bf 0.5570 & \bf 0.5605 & \bf 0.7522 \\
\bottomrule
\end{tabular}
}
\end{minipage} 
\end{tabular}
\end{table}

{\bf Training Overhead:} We empirically compare the training overhead of MaxCR with that of competing multimodal learning baselines under the same experimental setting, including Naive, PMR, MLA, ReconBoost, and LFM. The results are shown in Figure \ref{fig:train-time}, where training times are reported in hours. It can be observed that MaxCR achieves the best accuracy with limited training time.

{\bf Convergence:} In Figure \ref{fig:loss}, we present the convergence behavior of MaxCR on CREMAD dataset during training. It can be observed that the loss gradually converges as training proceeds without getting stuck. In addition, Figure \ref{fig:accuracy} reports the accuracy throughout training. We can observe that the accuracy exhibits a similar convergence trend.

\section{Conclusion} \label{sec:conclusion}
In this paper, we revisit modality imbalance in multimodal classification from the perspective of predictive certainty. We show that optimization yields asymmetric certainty gains: strong modalities become increasingly confident, whereas weak modalities struggle to translate optimization progress into predictive certainty. We characterize this behavior as optimization-certainty decoupling, which emerges in unimodal learning and persists in multimodal learning. Motivated by this insight, we propose multimodal Max Confidence Regularization (MaxCR), which dynamically regulates modality-specific predictions. Specifically, MaxCR uses a nonlinear sparsity measure to monitor semantic confidence and applies max suppression and max excitation to regularize the top-1 confidence of strong and weak modalities, respectively. To this end, MaxCR encourages confident yet not overconfident predictions, improving multimodal fusion. Experiments on widely used datasets demonstrate the effectiveness of MaxCR. Overall, our results highlight predictive certainty as an important yet underexplored dimension of modality imbalance and establish confidence-aware intervention as a simple yet effective complement to balance optimization strategies.

{\bf Limitations: } Our proposed method MaxCR focuses on multimodal classification. Extending it to other multimodal tasks may require task-specific confidence measures and regularization objectives. In addition, our theoretical analysis provides a tractable characterization of optimization-certainty decoupling, while a more general treatment of certainty evolution throughout optimization warrants further investigation. We leave a more comprehensive investigation as future work.



\bibliography{iclr2027_conference}
\bibliographystyle{iclr2027_conference}

\clearpage
\appendix
\renewcommand{\thefigure}{\Roman{figure}}
\renewcommand{\thetable}{\Roman{table}}

\makeatletter
\renewcommand{\theHtable}{A\arabic{table}}
\renewcommand{\theHfigure}{A\arabic{figure}}
\makeatother

\setcounter{figure}{0}
\setcounter{table}{0}
\setcounter{theorem}{0} 
\setcounter{lemma}{0}
\setcounter{assumption}{0}
\setcounter{definition}{0}
\setcounter{corollary}{0}
\setcounter{proposition}{0}

\colorlet{tocblue}{blue!60!white}

\startcontents[appendix]

\section*{Appendix Contents}

\begingroup
\color{tocblue}
\printcontents[appendix]{}{1}{%
    \setcounter{tocdepth}{2}%
}
\endgroup

\clearpage

\section{Theoretical Details and Proofs} \label{app:proof}

\subsection{Preliminaries and Assumptions}
For clarity, we consider two modalities. Let the parameters of the $r$-th modality be $\theta^r\in\RB^{d}$. The network produces logits $\z(\x^r;\theta)\in\RB^{C}$, with predictions $\p=\mathrm{Softmax}(\z)$. The Jacobian is defined as $\J_\theta(x^r)=\partial \z/\partial\theta\in\RB^{C\times d}$, and the sample-wise neural tangent kernel (NTK) is $K_\theta(\x^r,{\x^r}')=\J_\theta(\x^r)J_\theta({\x^r}')^\top$. Define the zero-sum subspace as $\SM_0=\{\v\in\RB^C:\mathbf 1^\top \v=0\}$, $ \Pi_0=\I-\tfrac1K\mathbf 1\mathbf 1^\top$. We further define $\g_{ce}=\nabla_\theta\LM_{CE}$, $\g_{e}=\nabla_\theta H$ and certainty $\CM(\theta)\triangleq- \EB_x[H(\p)]$.

\begin{assumption}[Modality Ambiguity]\label{ass:app_Ambiguity}
Let $\q^r(\x^r)=\PB(\y\mid \x^r)\in\Delta^{C-1}$ denote the Bayes posterior of $r$-modality. There exist constants $H_s<H_w$ such that $\EB_{\x}[H(\q^{a})]\le H_s$, $\EB_{\x}[H(\q^{v})]\ge H_w$.
Equivalently, $H(\Y\mid \X^{a})\ll H(\Y\mid \X^{v})$, i.e., the strong modality induces nearly deterministic labels, whereas the labels conditioned on the weak modality are intrinsically ambiguous.
\end{assumption}

\begin{assumption}[Regularity]\label{ass:app_Regularity}
There exist constants $B,\mu$ such that, for all $x$, $|\J_\theta(\x)|2\le B$, and the kernel is well-conditioned on $\SM_0$: $\mu\Pi_0\preceq \Pi_0K\theta(x,x)\Pi_0 \preceq L\Pi_0$, $\mu>0$.
\end{assumption}

The population cross-entropy loss admits the decomposition $\LM_{CE}(\theta)=\EB_\x[H(\q)]+\EB_\x[\mathrm{KL}(\q\|\p)]$. We define the excess risk as $\EM(\theta)=\EB_x[\mathrm{KL}(\q\|\p)]$.

Under gradient flow $\dot\theta=-\eta \g_{ce}$, we have the identity 
\begin{equation}
\begin{aligned}
\dot \CM(\theta)=\eta\langle \g_e,\g_{ce}\rangle.
\end{aligned}
\end{equation}
Let $\a(\p)=\log \p+H(\p)\mathbf 1$, which satisfies $\langle \p,\a\rangle=0$. Since $\nabla_\z H=-\p\odot \a$ and $\nabla_\z\ell=\p-\y$, define 
\begin{equation}
\begin{aligned}
\Psi(\p,\y) \triangleq \langle \nabla_\z H, \nabla_\z\ell\rangle=\p_y\a_y-\Gamma(\p),
\end{aligned}
\end{equation}
where $\Gamma(\p)=\sum_j \p_j^2\a_j=\mathrm{Cov}_{j\sim \p}(\p_j,\log \p_j)\ge 0$. $\Psi$ is the sample-wise certainty transfer coefficient. Under an isotropic kernel, $\dot \CM=\eta \EB[\Psi]$. Its sign is controlled by a candidate-set criterion determined solely by the predictive distribution.

\begin{definition}[Candidate Class Set]
$\CB(\p)=\{j:\ \p_j>e^{-H(\p)}\}$. We always have $\arg\max_j \p_j\in\CB(\p)$.
\end{definition}

The mini-batch gradient $\hat \g=\frac1N\sum_i\J_i^\top(\p_i-\y_i)$
satisfies $\EB\hat \g=\g_{ce}$ and $\EB\|\hat \g\|^2=\|\g_{ce}\|^2+\frac1N\mathrm{tr}\Sigma$. To handle arbitrary \(C\ge2\) in a unified manner, define
\begin{equation}
\begin{aligned}
F_C(e)\triangleq H_b(e)+e\log(C-1), e\in[0,1-1/C],
\end{aligned}
\end{equation}
and its inverse
\begin{equation}
\begin{aligned}
\phi_C\triangleq F_C^{-1}:[0,\log C]\to[0,1-1/C].
\end{aligned}
\end{equation}
Over this interval, $F_C$ is monotonically increasing and concave, and hence $\phi_C$ is monotonically increasing and convex.

We further define the effective certainty transfer rate:
\begin{equation}
\begin{aligned}
r_N(\theta)\triangleq\frac{|\langle \g_e,\g_{ce}\rangle|}{\|\g_e\|\sqrt{\EB\|\hat \g\|^2}}.
\end{aligned}
\end{equation}

\subsection{Auxiliary Lemmas and Proofs}

We have the following basic identity:

\begin{lemma}\label{lemma:app_1}
We have $\nabla_\z H(\p)=-\p\odot \a$, where $\a=\log \p+H(\p)\mathbf 1$. Moreover, $\langle \p,\a\rangle=0$, $\mathbf 1^\top\nabla_\z H=0$.
\end{lemma}

\begin{proof} 
Using $\partial \p_j/\partial \z_c=\p_j(\delta_{jc}-\p_c)$, we obtain
\begin{equation}
\begin{aligned}
\frac{\partial H}{\partial \z_c}&=-\sum_j\frac{\partial \p_j}{\partial \z_c}(\log \p_j+1) \\
&=-\p_c(\log \p_c+1)+\p_c\sum_j \p_j(\log \p_j+1) \\
&=-\p_c(\log \p_c+H).
\end{aligned}
\end{equation} 
Hence, $\nabla_\z H(\p)=-\p\odot \a$. Furthermore, $\langle \p, \a\rangle=\sum_j \p_j\log \p_j+H=0$, which implies $\mathbf 1^\top\nabla_\z H=-\langle \p, \a\rangle=0$.

This completes the proof.
\end{proof}

\begin{lemma}\label{lemma:app_2}
$\Gamma(\p)=\sum_j\p_j^2\a_j=\mathrm{Cov}_{j\sim \p}(\p_j,\log \p_j)\ge0$, with equality if and only if $\p$ is uniform over its support.
\end{lemma}

\begin{proof} 
Let $j\sim \p$. Since $\mathbb E_\p[\log \p_j]=-H$ and $\EB_\p[\p_j]=\|\p\|_2^2$, we have $\mathrm{Cov}(\p_j,\log \p_j)=\sum_j\p_j^2\log \p_j+H\|\p\|_2^2=\sum_j\p_j^2\a_j=\Gamma$.

Since larger $\p_j$ always corresponds to larger $\log \p_j$, the Chebyshev association inequality (FKG) gives a nonnegative covariance. Equality requires $\p_j$ to be almost surely constant.

This completes the proof.
\end{proof}

\begin{lemma}\label{lemma:app_3}
$e^{-H(\p)}\le\max_j\p_j$. Hence, \(\CB(\p)\neq\varnothing\) and contains the top-1 class.
\end{lemma}

\begin{proof} 
Since $H=\sum_j\p_j\log(1/\p_j)\ge\log(1/\p_{\max})$,
the result follows.

This completes the proof.
\end{proof}

\subsection{Theorems, Proofs and Corollaries}

\begin{theorem}[Certainty upper bound]\label{theorem:app_upperbound}
Under Assumption \ref{ass:app_Ambiguity}, for any $\theta$, let $\bar\varepsilon=\sqrt{\EM(\theta)/2}\le 1-1/C$. Then
$\CM(\theta)\le-\EB_x[H(\q)]+\bar\varepsilon\log(C-1)+H_b(\bar\varepsilon)$ where $H_b$ denotes the binary entropy function. In particular, any population minimizer of the cross-entropy loss satisfies $\CM(\theta^\star)=-\EB_\x[H(\q)]$.
\end{theorem}

\begin{proof} 
Pointwise, 
\begin{equation}
\begin{aligned}
\EB_{\y\sim \q}[-\log \p_y]=H(\q)+\mathrm{KL}(\q\|\p).
\end{aligned}
\end{equation} 
Taking the expectation over $\x$ gives $\LM_{CE}=\EB[H(\q)]+\EM(\theta)$. Let $T(\x)=\tfrac12\|\p(\x)-\q(\x)\|_1$. By Pinsker's inequality,
\begin{equation}
\begin{aligned}
T\le\sqrt{\mathrm{KL}(\q\|\p)/2}.
\end{aligned}
\end{equation}
By the Fannes–Audenaert continuity bound for entropy, when $T\le1-\frac1C$, we have
\begin{equation}
\begin{aligned}
|H(\p)-H(\q)|\le T\log(C-1)+H_b(T).
\end{aligned}
\end{equation}
Taking the expectation over \(x\), and noting that $t\mapsto t\log(C-1)+H_b(t)$ is concave on \([0,1-1/C]\), Jensen's inequality together with
\begin{equation}
\begin{aligned}
\EB[T]\le\sqrt{\EB[\mathrm{KL}]/2}=\bar\varepsilon,
\end{aligned}
\end{equation}
(Cauchy–Schwarz + Jensen) yields
\begin{equation}
\begin{aligned}
\EB[H(\q)]-\EB[H(\p)]\le\bar\varepsilon\log(C-1)+H_b(\bar\varepsilon).
\end{aligned}
\end{equation}
Rearranging the terms gives the result. If $\theta$ is a population minimizer and the model is realizable, then $\p=\q$ almost everywhere, and therefore
\begin{equation}
\begin{aligned}
\CM(\theta^\star)=-\EB[H(\q)].
\end{aligned}
\end{equation}

This completes the proof.
\end{proof}

\begin{corollary}[Asymmetric Gain]\label{corollary:app_AsymmetricGain}
Under Assumption \ref{ass:app_Ambiguity}, the certainty attainable under cross-entropy training satisfies $\CM^{\a}\geq -H_s-o(1)$ and $\CM^{v}\le-H_w+o(1)$. Hence, the confidence gap is at least $H_w-H_s$, and this gap is independent of the optimization algorithm or gradient reweighting scheme.
\end{corollary}

\begin{theorem}[Certainty Gain]\label{theorem:app_CertaintyGain}
If $\CB(\p)=\{\y\}$, i.e., the prediction is correct and the candidate class is unique, then $\Psi(\p,\y)\ge \p_y(1-\p_y)\a_y>0$. If $\y\notin\CB(\p)$, then $\Psi(\p,\y)\le-\Gamma(\p)\le0$. If the label follows an ambiguous posterior $y\sim \q$, then $\EB_{y\sim \q}\Psi=\langle \q-\p, \p\odot \a\rangle$.In particular, when $\q$ is uniform, $\EB_{y\sim \q}\Psi=-\Gamma(\p)\le0$.
\end{theorem}

\begin{proof} 
By Lemma \ref{lemma:app_1},
\begin{equation}
\begin{aligned}
\Psi=\langle -\p\odot \a, \p-\y\rangle=\p_y\a_y-\Gamma(\p).
\end{aligned}
\end{equation}

Suppose $\CB(\p)=\{\y\}$. Then $\a_y=\log \p_y+H\ge0$. For any $j\neq y$,
\begin{equation}
\begin{aligned}
\p_j\le e^{-H}\Rightarrow \a_j\le0\Rightarrow \p_j^2\a_j\le0.
\end{aligned}
\end{equation}
Therefore,
\begin{equation}
\begin{aligned}
\Gamma=\p_y^2\a_y+\sum_{j\neq y}\p_j^2\a_j\le \p_y^2\a_y,
\end{aligned}
\end{equation}
and hence
\begin{equation}
\begin{aligned}
\Psi\ge \p_y\a_y-\p_y^2\a_y=\p_y(1-\p_y)\a_y.
\end{aligned}
\end{equation}
By Lemma \ref{lemma:app_3}, $\p_y=\p_{\max}$. If $\p$ is non-uniform, then $\a_y>0$, and the inequality is strict.

If $\y\notin\CB(\p)$, then $\a_y\le0$, so $\p_y\a_y\le0$. By Lemma \ref{lemma:app_2}, $\Gamma\ge0$, and therefore
\begin{equation}
\begin{aligned}
\Psi\le-\Gamma\le0.
\end{aligned}
\end{equation}

We have
\begin{equation}
\begin{aligned}
\EB_{y\sim \q}\Psi=\sum_j\q_j\p_j\a_j-\Gamma=\langle \q,\p\odot \a\rangle-\langle \p,\p\odot \a\rangle=\langle \q-\p,\ \p\odot \a\rangle.
\end{aligned}
\end{equation}
For $\q=\u=\mathbf 1/C$,
\begin{equation}
\begin{aligned}
\langle \u,\p\odot \a\rangle=\frac1C\langle\mathbf 1,\p\odot \a\rangle=0,
\end{aligned}
\end{equation}
so
\begin{equation}
\begin{aligned}
\EB\Psi=-\Gamma\le0.
\end{aligned}
\end{equation}

This completes the proof.
\end{proof}

\begin{corollary}[Stability]\label{corollary:app_Stability}
For a general posterior $\q$, $\EB_{y\sim \q}\Psi\le-\Gamma(\p)+\omega(\p)\|\q-\u\|_1$, where $\omega(\p)=\max_j\p_j|\a_j|\le e^{-1}+\log C$. Thus, as the posterior approaches the uniform distribution, the certainty gain is increasingly driven toward non-positive values.
\end{corollary}

\begin{theorem}[Optimization-Certainty Decoupling]\label{theorem:app_Optimization-CertaintyDecoupling}
Under Assumptions \ref{ass:app_Ambiguity} and \ref{ass:app_Regularity}, we have $\mathrm{tr}\Sigma\ge\mu \EB_\x\left[1-\|\q\|_2^2\right]\ge\mu \phi_C \left(\EB_\x[H(\q)]\right)$ and $\|\g_{ce}\|\le B\sqrt{2\EM(\theta)}$. Therefore, for the weak modality, if $\EB_\x[H(\q^v)]\ge H_w>0$, then, as $\EM(\theta)\to0$,
$\dot \CM\to0, \EB\|\hat \g\|^2\ge\frac{\mu\phi_C(H_w)}{N}>0$, and
$r_N(\theta)\le\frac{B\sqrt{2\EM(\theta)}}{\sqrt{2B^2\EM(\theta)+\mu\phi_C(H_w)/N}}\to0$.
\end{theorem}

\begin{proof} 
Let the sample-wise stochastic gradient be $\g(\x,\y)=\J_\theta(\x)^\top(\p-\y)$. By the law of total covariance,
\begin{equation}
\begin{aligned}
\Sigma=\mathrm{Cov}(\g)=\EB_\x \left[\mathrm{Cov}_{y\sim \q}(\g)\right]+\mathrm{Cov}_\x\!\left(\EB_\y \g\right)\succeq\EB_\x \left[\J^\top(\mathrm{diag}(\q)-\q\q^\top)\J\right].
\end{aligned}
\end{equation}
Let
\begin{equation}
\begin{aligned}
A(\q)=\mathrm{diag}(\q)-\q\q^\top\succeq0.
\end{aligned}
\end{equation}
Since $A(\q)\mathbf1=0$, its range is contained in $\SM_0$, and $\mathrm{tr}A(\q)=1-\|\q\|_2^2$. By Assumption 2,
\begin{equation}
\begin{aligned}
\mathrm{tr}\left(\J^\top A(\q)\J\right)&=\mathrm{tr}\left(A(\q)K_\theta(\x,\x)\right)\\
&=\mathrm{tr}\left(A(\q)\Pi_0K_\theta(\x,\x)\Pi_0\right)\\
&\ge \mu \mathrm{tr}A(\q)\\
&=\mu(1-\q\|_2^2).
\end{aligned}
\end{equation}
Taking the expectation over $x$ gives
\begin{equation}
\begin{aligned}
\mathrm{tr}\Sigma\ge\mu \EB_\x\left[1-\|\q\|_2^2\right].
\end{aligned}
\end{equation}
We next convert entropy directly into a lower bound on the error mass that holds for any $C\ge2$. Let
\begin{equation}
\begin{aligned}
\q_{\max}=\max_j\q_j, P_e=1-\q_{\max}\in[0,1-1/C].
\end{aligned}
\end{equation}
Since
\begin{equation}
\begin{aligned}
\|\q\|_2^2=\sum_j\q_j^2\le \q_{\max}\sum_j\q_j=\q_{\max},
\end{aligned}
\end{equation}
we can obtain
\begin{equation}
\begin{aligned}
1-\|\q\|_2^2\ge P_e.
\end{aligned}
\end{equation}
On the other hand, grouping the most probable class against all remaining classes yields
\begin{equation}
\begin{aligned}
H(\q)\le H_b(P_e)+P_e\log(C-1)=F_C(P_e).
\end{aligned}
\end{equation}
Its derivatives are
\begin{equation}
\begin{aligned}
F_C'(e)=\log\frac{(C-1)(1-e)}{e}\ge0, F_C''(e)=-\frac1{e(1-e)}<0.
\end{aligned}
\end{equation}
Thus, $F_C$ is monotonically increasing and concave on this interval, and its inverse $\phi_C=F_C^{-1}$ is monotonically increasing and convex. Therefore, $P_e\ge\phi_C(H(q))$, and pointwise, $1-\|\q\|_2^2 \ge \phi_C(H(\q))$. Applying Jensen's inequality gives
\begin{equation}
\begin{aligned}
\EB_x\!\left[1-\|\q\|_2^2\right]\ge\EB_\x[\phi_C(H(\q))]\ge\phi_C(\EB_\x[H(\q)]).
\end{aligned}
\end{equation}
Hence,
\begin{equation}
\begin{aligned}
\boxed{\mathrm{tr}\Sigma\ge\mu \phi_C (\EB_\x[H(\q)])}.
\end{aligned}
\end{equation}

On the other hand,
\begin{equation}
\begin{aligned}
\g_{ce}=\EB_\x[\J^\top(\p-\q)].
\end{aligned}
\end{equation}
Using $\|\J\|_2\le B$, $\|\v\|_2\le\|\v\|_1$, Pinsker's inequality, and Jensen's inequality,
\begin{equation}
\begin{aligned}
\|\g_{ce}\|&\le B\,\EB_\x\|\p-\q\|_2\\ 
&\le B\,\EB_\x\|\p-\q\|_1\\ 
&\le B\sqrt{2\EB_\x[\mathrm{KL}(\q\|\p)}]=B\sqrt{2\EM(\theta)}.
\end{aligned}
\end{equation}
To show that the certainty gain vanishes, it suffices to note that, for fixed $C$, the entropy gradient is uniformly bounded. By Lemma \ref{lemma:app_1}, we have
\begin{equation}
\begin{aligned}
\nabla_\z H=-\p\odot(\log \p+H\mathbf1).
\end{aligned}
\end{equation}
Using $|\u\log \u|\le e^{-1}$ and $H(\p)\le\log C$, there exists a finite constant $M_C$ depending only on $C$ such that $\|\nabla_\z H\|_2\le M_C$. Therefore,
\begin{equation}
\begin{aligned}
\|\g_e\|=\left\|\EB_\x[\J^\top\nabla_\z H]\right\|\le BM_C.
\end{aligned}
\end{equation}
Under gradient flow $\dot\theta=-\eta \g_{ce}$,
\begin{equation}
\begin{aligned}
|\dot \CM|=\eta|\langle \g_e,\g_{ce}\rangle|\le\eta B^2M_C\sqrt{2\EM(\theta)}\to0.
\end{aligned}
\end{equation}
If the weak modality satisfies
\begin{equation}
\begin{aligned}
\EB_\x[H(\q^v)]\ge H_w>0,
\end{aligned}
\end{equation}
then, because $\phi_C$ is strictly increasing and $\phi_C(0)=0$, $\phi_C(H_w)>0$. Hence, $\mathrm{tr}\Sigma\ge\mu\phi_C(H_w)$. For an i.i.d. mini-batch of size $N$,
\begin{equation}
\begin{aligned}
\EB\|\hat \g\|^2=\|\g_{ce}\|^2+\frac1N\mathrm{tr}\Sigma\ge\frac{\mu\phi_C(H_w)}{N}>0.
\end{aligned}
\end{equation}
Finally, by the Cauchy–Schwarz inequality, we obtain
\begin{equation}
\begin{aligned}
r_N(\theta)=\frac{|\langle \g_e,\g_{ce}\rangle|}{\|\g_e\|\sqrt{\EB\|\hat \g\|^2}}\le\frac{B\sqrt{2\EM(\theta)}}{\sqrt{2B^2\EM(\theta)+\mu\phi_C(H_w)/N}}\to0.
\end{aligned}
\end{equation}

This completes the proof.
\end{proof}

\section{Analysis of MaxCR}
In this section, we analyze the effectiveness and conservativeness of peak regularization. Our analysis aims to provide a local explanation in an idealized setting, while a rigorous treatment that directly accounts for practical DNN training is left to future work.

\subsection{Certainty Effect}

\begin{proposition}[Certainty-Active Component]\label{proposition:app_Certainty-Active}
Let $m=\arg\max_j\p_j$. For $\LM_{exc}=\lambda\left(\frac1C-\p_m\right)$, the induced sample-wise rate of certainty change is $\dot \CM_{exc} = \lambda \p_m \Psi(\p,m)\ge\lambda \p_m^2(1-\p_m) \a_m\ge0$, and is strictly positive when $\CB(\p)=\{m\}$. The loss $\LM_{sup}$ induces an equal-magnitude change with the opposite sign. This result is independent of the label $\y$ and posterior $\q$.
\end{proposition}

\begin{proof} 
Since $\nabla_\z \p_m=\p_m(\e_m-\p)$, gradient descent gives
\begin{equation}
\begin{aligned}
\dot \z=-\nabla_\z\LM_{exc}=\lambda \p_m(\e_m-\p),
\end{aligned}
\end{equation}
where $\e_m$ is a one-hot vector. Therefore,
\begin{equation}
\begin{aligned}
\dot \CM_{exc}&=-\langle\nabla_\z H,\dot \z\rangle \\
&=\lambda \p_m \langle \p\odot \a, \e_m-\p\rangle \\
&=\lambda \p_m(\p_m\a_m-\Gamma) \\
&=\lambda \p_m\Psi(\p,m).
\end{aligned}
\end{equation}
By Lemma \ref{lemma:app_3},
\begin{equation}
\begin{aligned}
\p_m=\p_{\max}\ge e^{-H} \Rightarrow \a_m\ge0.
\end{aligned}
\end{equation}
If $\CB(\p)=\{m\}$, then by Theorem \ref{theorem:app_CertaintyGain}, 
\begin{equation}
\begin{aligned}
\Psi(\p,m)\ge \p_m(1-\p_m)\a_m>0.
\end{aligned}
\end{equation}
The same argument for $\LM_{sup}$ yields the opposite sign. 

This completes the proof.
\end{proof}

\subsection{Comparison with Entropy and Conservativeness}

\noindent{\bf Difference from Entropy Minimization:} For $\LM_{EM}=\lambda H(\p)$, we have $\dot \CM_{EM} =\lambda\|\nabla_\z H\|^2=\lambda\sum_j\p_j^2\a_j^2$. Near a uniform prediction, $\a\to0$, and therefore $\dot \CM_{EM}=O(\|\a\|^2)$ is second-order small. In contrast, for $\dot \CM_{exc} = \lambda \p_m\Psi(\p,m)$, $\Psi$ is first-order in the margin.

\noindent{\bf Conservativeness:} When the top-1 class $m\neq y$, the excitation term reinforces an incorrect prediction. Since $\lambda\le 1$, $\p_y\le \p_m$, $\p_y\le 1-\p_m$, and $\p_m^{2}+\p_y^{2}\le\|\p\|_2^{2}\le \p_m$, we have 
\begin{equation}
\begin{aligned}
\frac{d}{dt}\log\frac{\p_m}{\p_y} &= -(1+\p_m-\p_y)+\lambda \p_m(1-\p_m+\p_y) \\
&\le -(1+\p_m-\p_y)+\p_m(1-\p_m+\p_y) \\
&\le \p_y - 1 \\
&\le -\p_m \\
&\le -1/C.
\end{aligned}
\end{equation}
Therefore, we can obtain a uniform negative upper bound.

Consider one step of gradient descent:
\begin{equation}
\begin{aligned}
\dot \p_y &= \p_y(1-\p_y-\lambda \p_m \p_y-\langle \p,(\e_y-\p)+\lambda \p_m\,(\e_m-\p)\rangle) \\
&=\p_y(1-\p_y-\lambda \p_m \p_y-(\p_y-\|\p\|_2^2)+\lambda \p_m(p_m-\|\p\|_2^2)).
\end{aligned}
\end{equation}
Therefore, we have
\begin{equation}
\begin{aligned}
\frac{\dot \p_y}{\p_y}=(1-2\p_y+\|\p\|_2^2)-\lambda \p_m(\p_y+\p_m-\|\p\|_2^2).
\end{aligned}
\end{equation}
Since $\|\p\|_2^2\le \p_m$, we have $\p_m(\p_y+\p_m-\|\p\|_2^2 \ge \p_m\p_y\ge0$. So, when $\lambda\le 1$, we have
\begin{equation}
\begin{aligned}
\frac{\dot \p_y}{\p_y}&\ge(1-2\p_y+\|\p\|_2^2)-\p_m(\p_y+\p_m-\|\p\|_2^2) \\
&\ge (1-\p_y)[(1-\p_y)-\p_m\p_y]+\p_m^{3} \\
&\ge (1-\p_y)^{3}+\p_m^{3} \\
&>0.
\end{aligned}
\end{equation}
That is, the odds of the incorrect top-1 class against the ground-truth class contract at a rate bounded away from zero uniformly over samples, and the ground-truth class probability strictly increases.

\section{Notation Definition}
We summarize the notation definition in Table \ref{tab:notation}.

\begin{table}[t]
\centering
\caption{Notation Definition}
\label{tab:notation}
\begin{tabular}{c|c}
\toprule
Notation   & Description                   \\\midrule
$\DM$ & Training dataset. \\
$\x^{a}_i/\x^{v}_i$     & Audio/video data point.             \\
$\y_i\in\{0,1\}^C$     & Category label of $i$-th data.\\
Superscript $r$& $r$-modality.             \\
$\x^{r}/\y$     & Data and category labels. \\
$M$        & The number of data points.  \\
$C$        & The number of category labels.\\
$f^r(\cdot)$  & Encoder of $r$ modality.    \\
$\theta^r$       & Parameters of $f^r(\cdot)$.\\
$\z$ & Logits of data. \\
$\h^r$ & Feature of $\x^r$. \\
$\psi^r(\cdot)$  & Classifier of $r$ modality.    \\ 
$\phi^r$       & Parameters of $\psi^r(\cdot)$.\\
$\p^r$           & Prediction of $\x^r$.           \\
$\Phi^r=\{\theta^r,\phi^r\}$ & Parameters set.\\
$\ell(\cdot)$    & Cross entropy loss.             \\
$\LM_{CE}(\cdot)$       & Cross entropy loss over training. \\
$s$        & Sparsity metric. \\
$\tau$  & Scaling factor.\\
$S$  & Sparsity score.\\
$\B$  & Mini-batch data.\\
$N$  & Mini-batch size.\\
$\sigma$  & Dead zone.\\
$\lambda$  & Modality discrepancy ratio.\\
$\LM_{sup}(\cdot)$       & Max suppression loss. \\
$\LM_{exc}(\cdot)$       & Max excitation loss. \\
$H(\cdot)$       & Entropy loss. \\
$\g_{ce}$       & Task gradient. \\
$\g_e$       & Entropy gradient. \\
$\J$       &  Jacobian. \\
$\EM$       & Excess risk. \\
$r_N(\theta)$       & Certainty-transfer ratio. \\
\bottomrule
\end{tabular}
\end{table}

\section{Experiment Details}
\subsection{Dataset Details}
\noindent{\bf CREMAD:} CREMAD \cite{CREMAD:journals/taffco/CaoCKGNV14} is an audio-visual dataset designed for speech emotion recognition. It comprises 7,442 video clips of 2$\sim$3 seconds from 91 actors speaking several short words. This dataset includes the six most common emotions.

\noindent{\bf KSounds:} KSounds \cite{Kinetics-Sound:conf/iccv/ArandjelovicZ17} dataset is a commonly used dataset containing 31 action categories that can be recognized visually and auditorily, which contains 19,000 10-second video clips from YouTube.

\noindent{\bf VGGSound:} VGGSound \cite{VGGSound:conf/icassp/ChenXVZ20} dataset is an audio-visual dataset in the wild, with nearly 200K 10-second video clips. Each sound-emitting object is also visible in the corresponding video clip in this dataset. After filtering out unavailable videos, 168,618 videos were selected for training and validation, and 13,954 videos for testing in experimental settings.

\noindent{\bf Twitter:} Twitter \cite{Twitter15:conf/ijcai/Yu019} dataset is a dataset designed for emotion recognition. This dataset consists of tweets collected from Twitter with three different labels: positive, negative, and neutral.

\noindent{\bf Sarcasm:} Sarcasm \cite{Sarcasm:conf/acl/CaiCW19} dataset is specifically designed for sarcasm detection tasks, containing a large amount of text data labeled as sarcastic or non-sarcastic, which typically draws from various sources, such as social media, news comments, and conversation data, with each entry clearly labeled.

\noindent{\bf NVGesture:} NVGesture \cite{NVGeasture:conf/cvpr/MolchanovYGKTK16} dataset is a multimodal dataset specifically designed for gesture recognition, containing three types of data modalities, i.e., RGB, optical flow~(OF), and Depth. This dataset includes 25 different gesture categories, covering a variety of common gestures.

\subsection{More Implementation Details}
We report implementation details such as the learning rate and optimizer in the main paper, and further provide additional details here. The batch size is set to be 64 for CREMAD dataset, 32 for KSounds, Twitter, and Sarcasm datasets, 16 for VGGSound dataset, and 2 for NVGesture dataset in training. Following \cite{AUG:conf/nips/Jiang2025aug}, we conduct experiments with three random seeds and report the mean and standard deviation. In addition, following \cite{MLA:conf/cvpr/ZhangYBY24,ReconBoost:conf/icml/CongHua24,LFM:conf/nips/0074WJ024,AUG:conf/nips/Jiang2025aug}, we obtain multimodal predictions by averaging the predictions from all modalities.

\subsection{Extension to More than Two Modalities}
Although we present our method in the paper using a bimodal setting for simplicity, MaxCR can be readily generalized to more modal scenarios with a simple strategy. Specifically, when computing the modality discrepancy ratio $\lambda$, we compare the sparsity score $S^r_t$ of the current modality against the average score $\Bar{S_t}$ over all modalities:
\begin{align}
\lambda=\left | S^r_t - \Bar{S_t}\right |.
\end{align}
The current modality is then suppressed or excited according to this ratio, which naturally extends MaxCR from the bimodal case to more general multimodal settings.

\section{Additional Experiments}
\subsection{Further Validation Experiments}
\noindent{\bf Experimental Details:} Unlike naive MML, which is trained solely with concatenation, we additionally introduce unimodal losses into multimodal learning as a baseline, following \cite{MMPareto:conf/icml/WeiH24,InfoReg:huang2025adaptive}. This uniform learning enhances unimodal gradient signals and partially alleviates modality imbalance. Moreover, in the validation experiments, introducing unimodal losses facilitates the investigation of the relationship between unimodal task loss and the entropy reduction direction in multimodal learning.

\begin{figure}[t]\centering
\includegraphics[width=.32\linewidth]{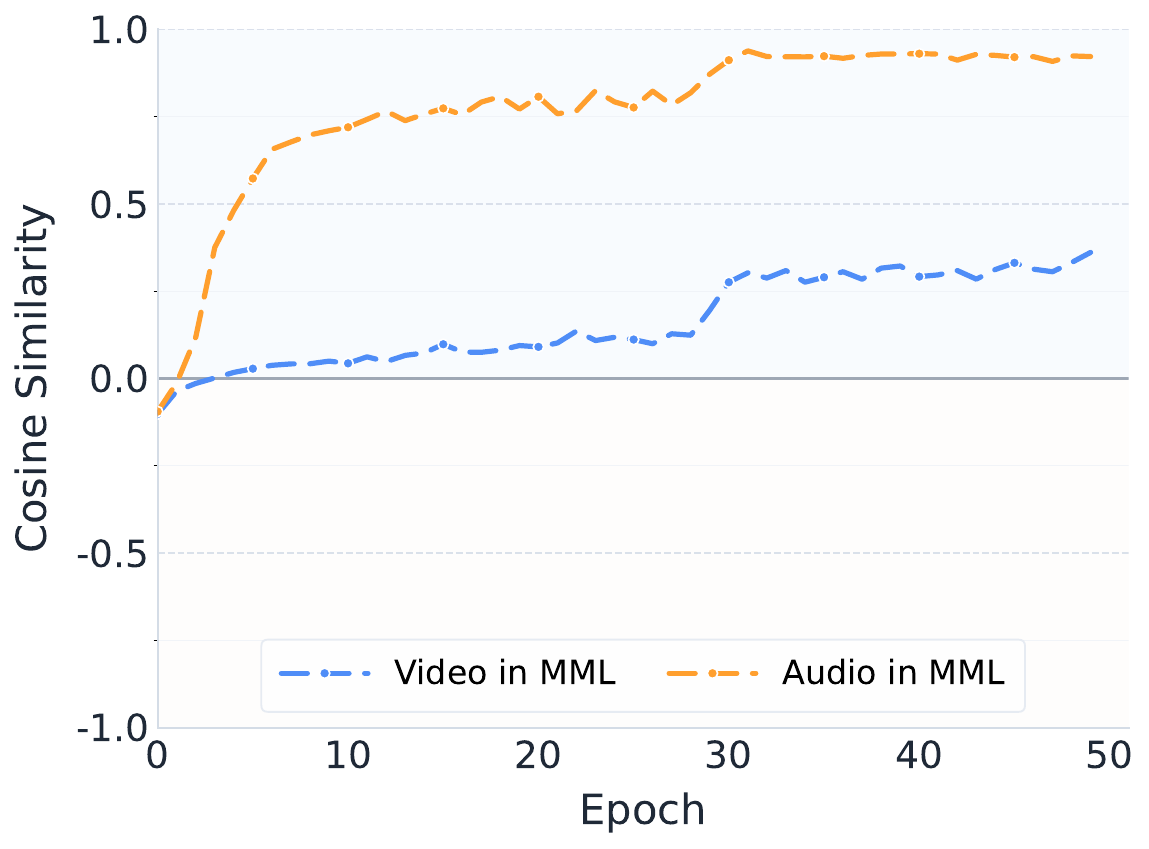}
\includegraphics[width=.32\linewidth]{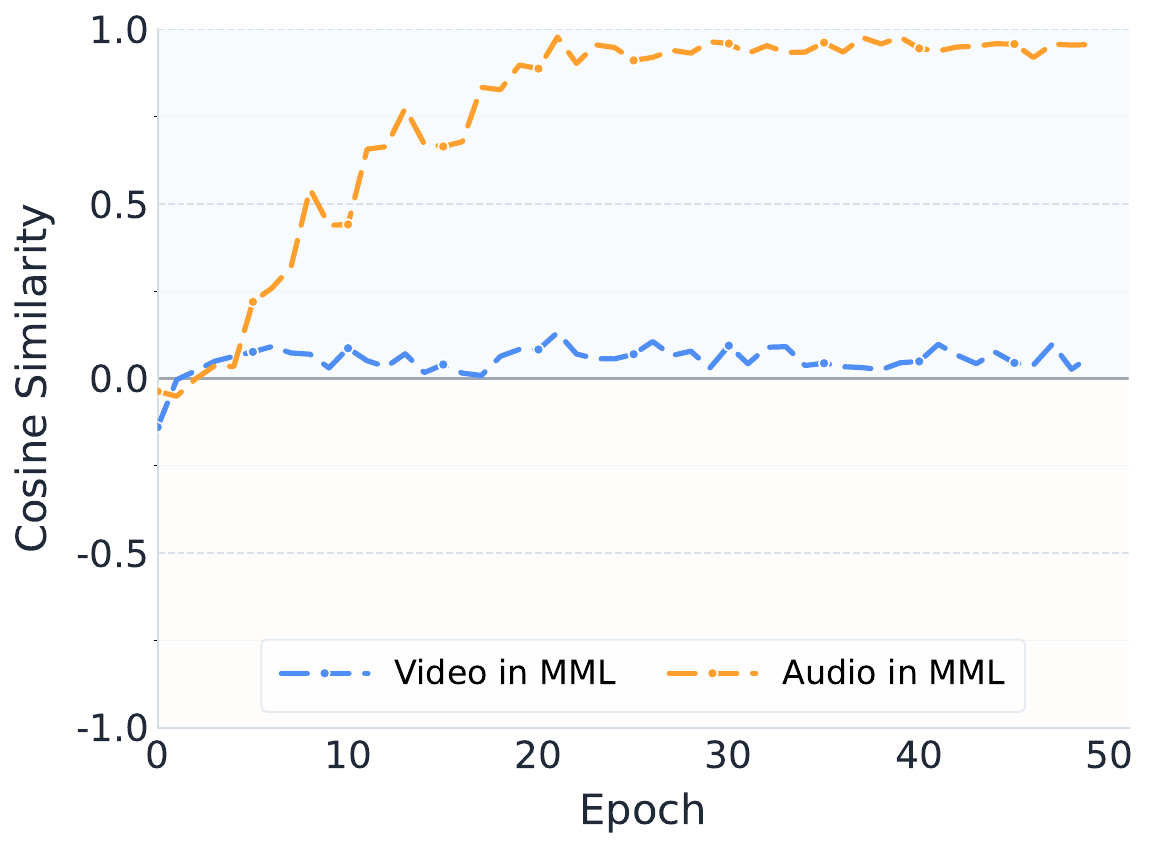}
\includegraphics[width=.32\linewidth]{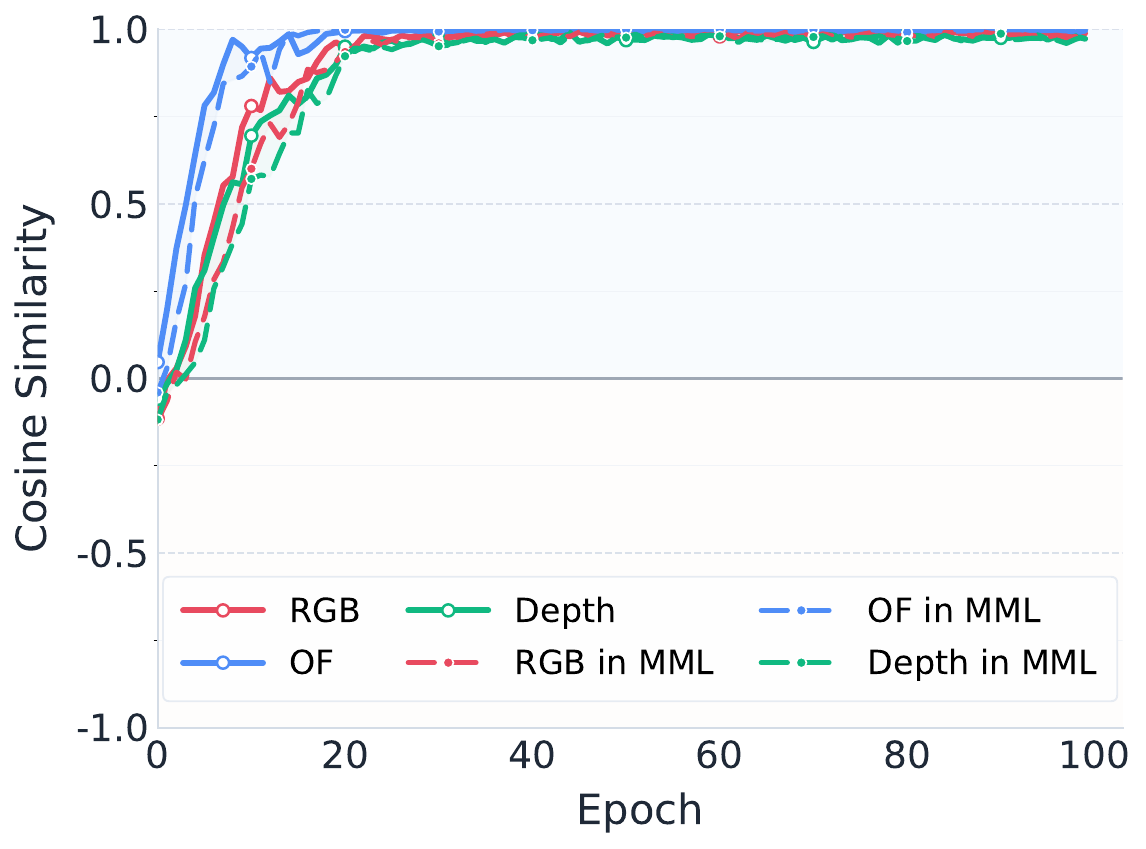} \\
\caption{Additional experiment between the task gradient and the entropy-reduced gradient on Sarcasm (left), Twitter (middle) and NVGesture (right) datasets.}
\label{fig:app_cs}
\end{figure}

\noindent{\bf More Results and Analysis:} We conduct experiments on image-text datasets, namely Sarcasm and Twitter. As shown in Figure \ref{fig:app_cs}, similar to the audio-video datasets, the image-text datasets also exhibit $r_N(\theta) \to0$. A natural question is whether the optimization-certainty decoupling also exists in relatively homogeneous multimodal data, where modality imbalance is less pronounced and the unimodal confidence gap is less evident. To investigate this issue, we further conduct experiments on NVGesture dataset, which consists of three visual modalities, i.e., RGB, optical flow (OF), and Depth. Compared with audio-video settings, these modalities are more homogeneous and exhibit weaker modality discrepancy. As shown in Figure \ref{fig:app_cs}, the optimization-certainty decoupling is not consistently observed throughout training on NVGesture dataset, and the discrepancy mainly appears in the early stage. Nevertheless, MaxCR still improves the performance. This gain may arise because MaxCR does not directly reshape the optimization dynamics, but instead corrects asymmetric confidence behavior. The result indicates that our method remains beneficial beyond the assumption in the analysis, and also generalizes to multimodal settings where modality imbalance is relatively mild.

\subsection{Selection of Monitoring Metric on CREMAD dataset}
Unlike conventional moving average smoothing that emphasizes historical information, multimodal rebalancing learning requires greater attention to the real-time learning state. Our smoothing of monitoring indicators aims to mitigate state instability in the early training stage. Therefore, the smoothing of the sparsity score progressively reduces the influence of historical information. Furthermore, to evaluate the effectiveness of the sparsity metric, we conduct experiments on CREMAD dataset. As shown in Table \ref{tab:sparsity}, the sparsity metric achieves the best performance.

\subsection{Incorporating Optimization Dynamics}
MaxCR focuses on correcting the asymmetric semantic confidence between strong and weak modalities, whereas methods such as OGM aim to rebalance optimization dynamics by modulating their gradient signals. These two mechanisms operate at different levels and are therefore compatible. To verify this compatibility, we further combine MaxCR with OGM, where OGM is used to regulate the gradient updates, and MaxCR is used to regularize their prediction confidence. This combined design allows multimodal learning to benefit from both balanced optimization dynamics and improved modality semantic confidence. We conduct experiments on KSounds dataset. As shown in Table \ref{tab:ogm_maxcr}, introducing OGM for gradient modulation can further improve model performance.

\begin{table}[t]
\centering
\begin{tabular}{cc}
\begin{minipage}{.44\textwidth}
\caption{Performance with different monitoring metrics on CREMAD dataset.}
\label{tab:sparsity}
\centering
\renewcommand\arraystretch{1.3}
\resizebox{\textwidth}{!}{
\begin{tabular}{c|ccc}
\toprule
Method     &  Audio & Video & Multimodal  \\\midrule
Confidence      & 0.6788 & 0.6613 & 0.8427 \\
Entropy     & 0.6815 & 0.6694 & 0.8468 \\
Sparsity (Our)      & \bf 0.6855 & \bf 0.6828 & \bf 0.8531 \\
\bottomrule
\end{tabular}
}
\end{minipage} &
\begin{minipage}{.50\textwidth}
\caption{Performance of MaxCR with OGM on KSounds dataset.}
\label{tab:ogm_maxcr}
\centering
\renewcommand\arraystretch{1.3}
\resizebox{\textwidth}{!}{
\begin{tabular}{c|ccc}
\toprule
Method     &  Audio & Video & Multimodal  \\\midrule
OGM      & 0.5013 & 0.3165 & 0.6582 \\
MaxCR      & 0.5570 & \bf 0.5605 & 0.7522 \\
MaxCR+OGM     & \bf 0.5651 & 0.5593 & \bf 0.7530 \\
\bottomrule
\end{tabular}
}
\end{minipage} 
\end{tabular}
\end{table}

\subsection{Analysis on Semantic Confidence}
To assess MaxCR's impact on semantic confidence, we evaluate the mean top-1 predictive confidence and average number of candidate classes (classes with probabilities exceeding the mean). Assuming $C$ total classes, we experiment on the CREMAD dataset ($C=6$). As reported in Table \ref{tab:conf}, the results demonstrate that MaxCR enhances confidence and reduces candidate classes for video modality, while moderately suppressing confidence for audio modality.

\begin{table}[!htbp]
\centering
\caption{Comparison of Acc, Top-1 confidence and candidate classes ($>1/C$).}
\label{tab:conf}
\begin{tabular}{r|c|cc|cc}
\toprule
\multirow{2}{*}{Method} 
& \multirow{2}{*}{ACC} 
& \multicolumn{2}{c|}{Top-1} 
& \multicolumn{2}{c}{$>1/C$} \\
\cmidrule(lr){3-4} \cmidrule(lr){5-6}
& & Audio & Video & Audio & Video \\
\midrule
Naive & 0.6361 & \bf 0.9234 & 0.6643 & 1.1694 & 1.6640 \\
OGM   & 0.6612 & 0.9232 & 0.6972 & 1.1788 & 1.5121 \\
MaxCR & \bf 0.8531 & 0.8768 & \bf 0.7783 & 1.2742 & 1.4892 \\
\bottomrule
\end{tabular}
\end{table}

\begin{figure*}[t] 
\centering
\begin{minipage}{.32\linewidth}
\centering
\includegraphics[width=\linewidth]{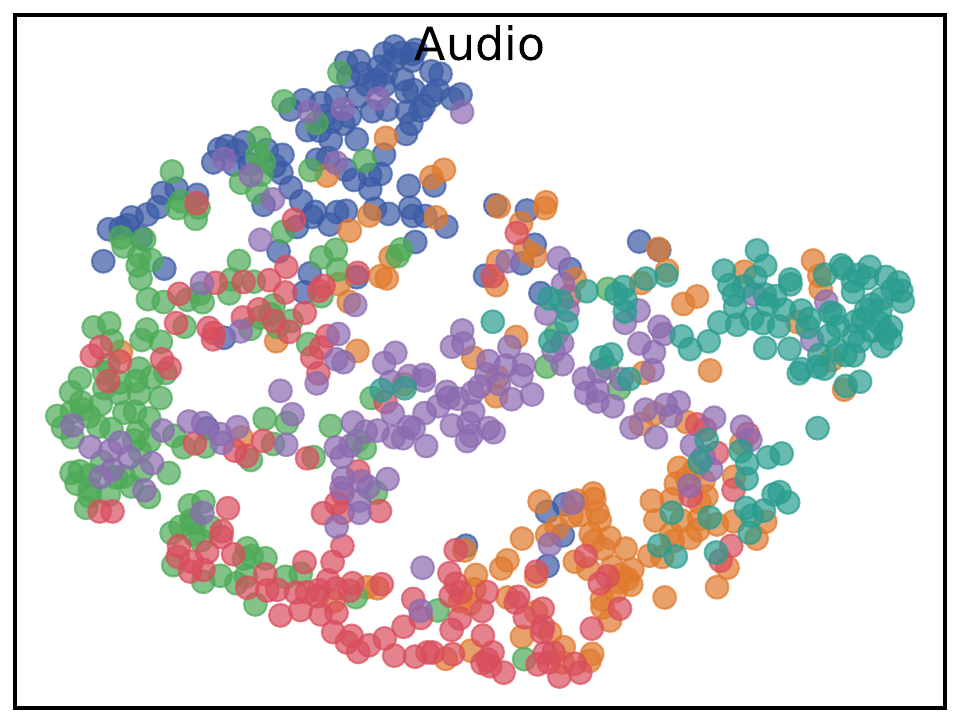}\\
{(a). Naive Audio.}
\end{minipage} 
\begin{minipage}{.32\linewidth}
\centering
\includegraphics[width=\linewidth]{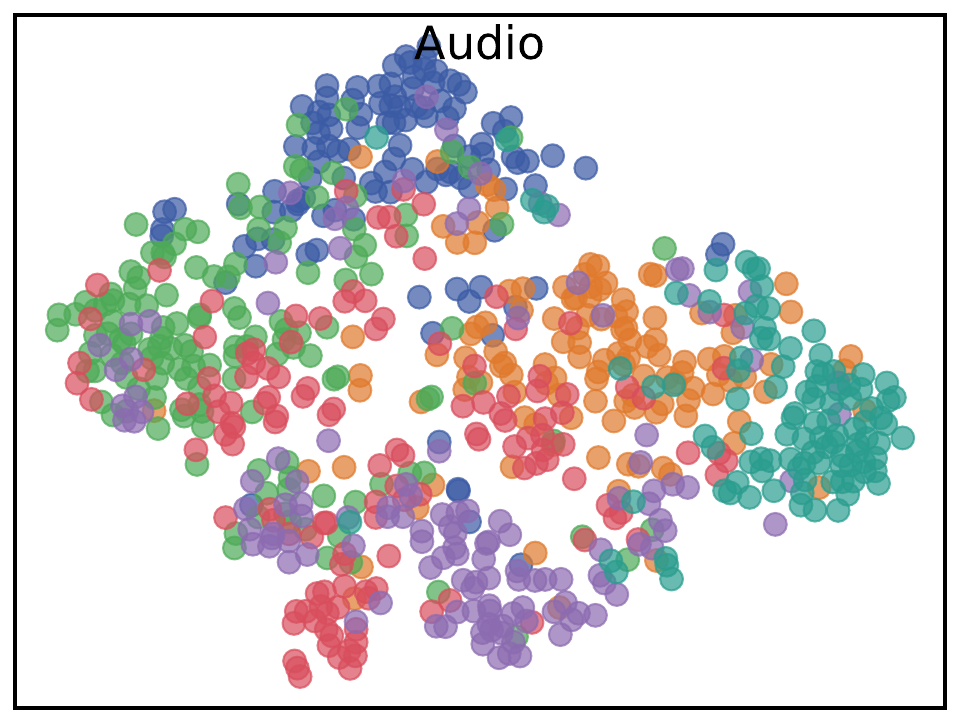}\\
{(b). LFM Audio.}
\end{minipage}
\begin{minipage}{.32\linewidth}
\centering
\includegraphics[width=\linewidth]{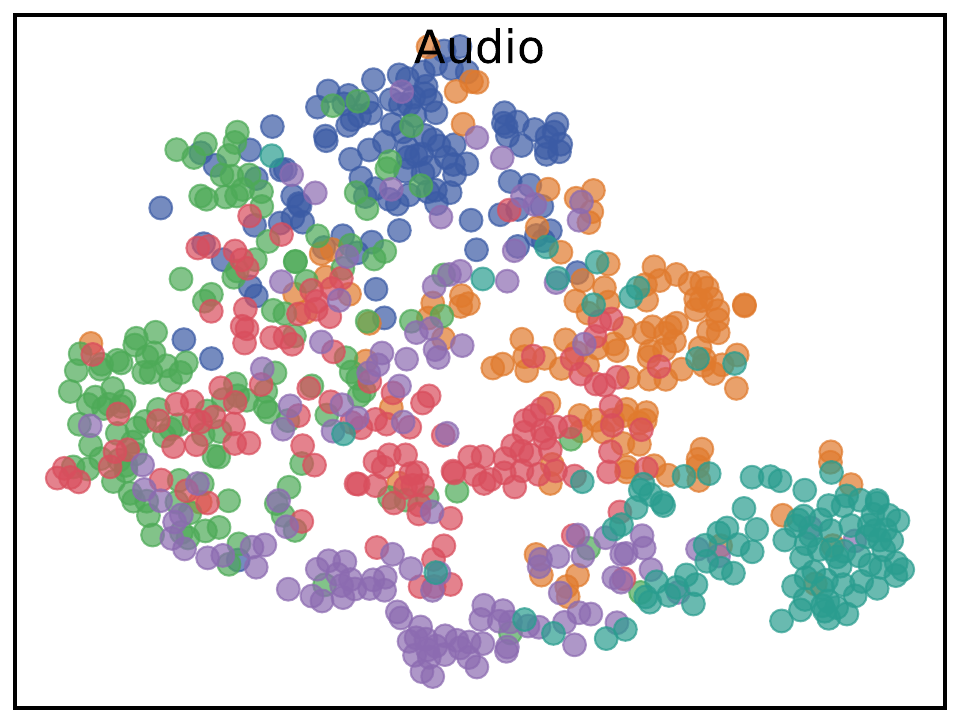}\\
{(c). MaxCR Audio.}
\end{minipage}\\
\begin{minipage}{.32\linewidth}
\centering
\includegraphics[width=\linewidth]{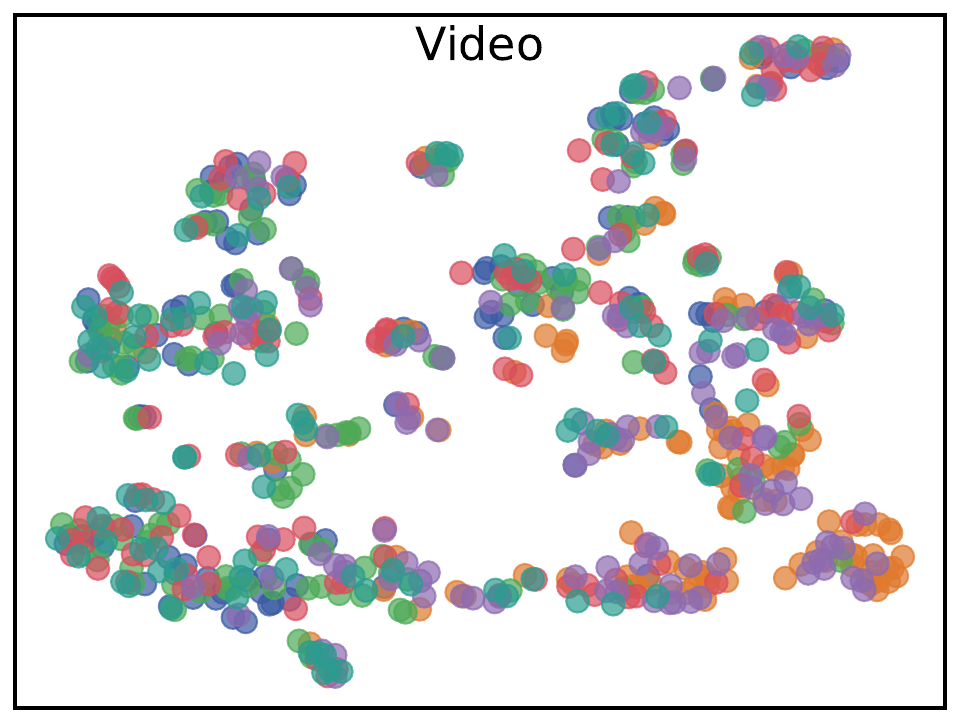}\\
{(d). Naive Video.}
\end{minipage} 
\begin{minipage}{.32\linewidth}
\centering
\includegraphics[width=\linewidth]{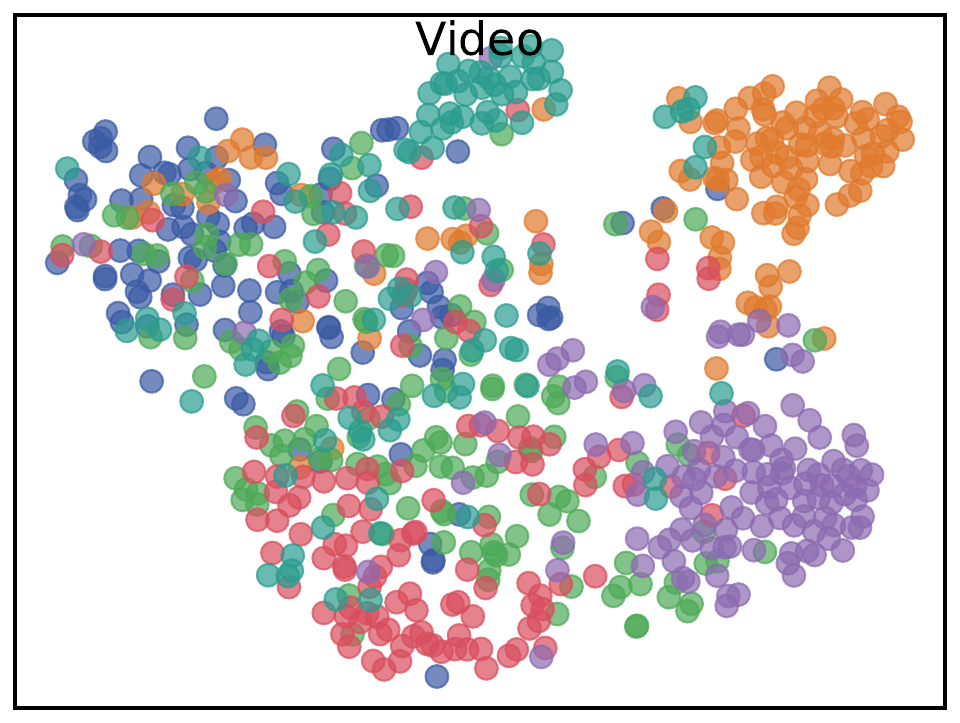}\\
{(e). LFM Video.}
\end{minipage}
\begin{minipage}{.32\linewidth}
\centering
\includegraphics[width=\linewidth]{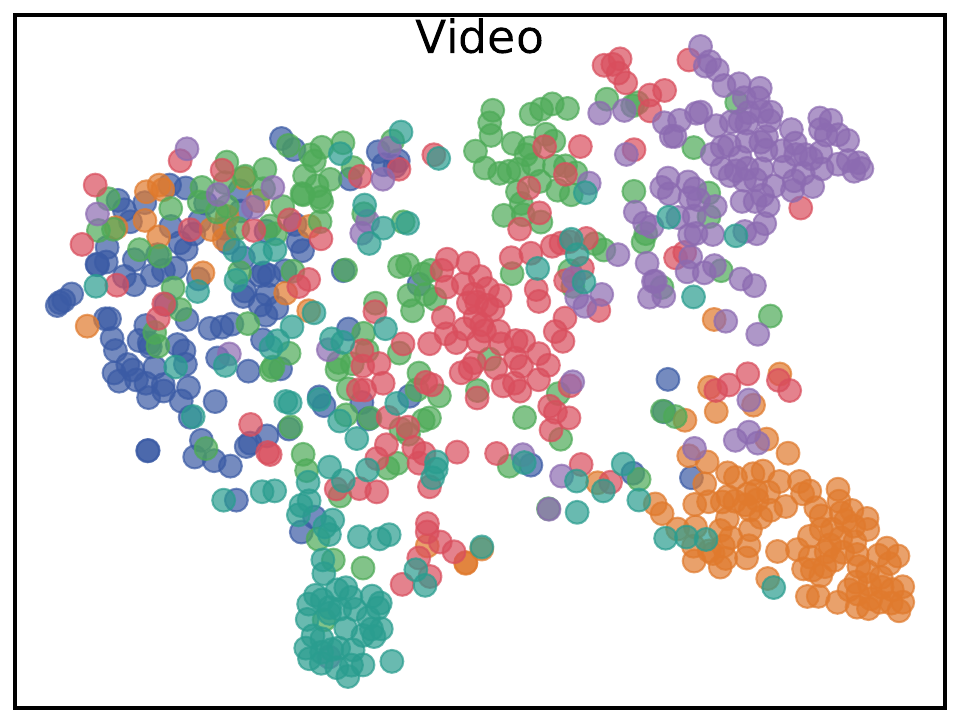}\\
{(f). MaxCR Video.}
\end{minipage}
\caption{The representations of audio and video modalities on CREMAD dataset by t-SNE \cite{tSNE:journal/jmlr/MaatenH08} across different methods are shown.}
\label{fig:visualization}
\end{figure*}

\subsection{Visualization Results}
In Figure \ref{fig:visualization}, we further analyze the proposed MaxCR via t-SNE \cite{tSNE:journal/jmlr/MaatenH08} feature visualization. Specifically, we compare the visualization results of naive multimodal learning (Naive), the label intervention method LFM, and the proposed method MaxCR. It can be observed that the visual features learned by Naive MML do not exhibit clear decision boundaries, which is consistent with its inferior performance. In contrast, LFM and MaxCR produce higher quality representations. Moreover, the video modality features learned by MaxCR form clearer clusters.

\end{document}